\documentclass[conference]{IEEEtran}
\usepackage{times}

\usepackage[numbers]{natbib}
\usepackage{multicol}
\usepackage[bookmarks=true]{hyperref}
\usepackage{graphicx}
\usepackage{xcolor}
\usepackage{float}
\usepackage{amsmath, amssymb, amsthm}
\usepackage{bm}
\usepackage[linesnumbered,ruled,vlined]{algorithm2e}

\usepackage{algorithmic}

\usepackage{caption}
\usepackage{lipsum}
\usepackage{svg}
\newtheorem{theorem}{Theorem}[section] 

\newtheorem{lemma}[theorem]{Lemma} 
\newtheorem{definition}[theorem]{Definition} 
\usepackage{amsmath}
\usepackage{diagbox}
\usepackage{booktabs}
\usepackage{multirow}
\usepackage{threeparttable}    
\usepackage{hyperref}

\begin{document}
\theoremstyle{definition}

\title{Robust Peg-in-Hole Assembly \\under Uncertainties via Compliant and Interactive Contact-Rich Manipulation}


\author{\authorblockN{Yiting Chen$^{1}$,
Kenneth Kimble$^{2}$,
Howard H. Qian$^{1}$, 
Podshara Chanrungmaneekul$^{1}$,
Robert Seney$^{2}$, 
Kaiyu Hang$^{1}$}
\authorblockA{$^{1}$Department of Computer Science, Rice University }
\authorblockA{$^{2}$U.S. National Institute of Standards and Technology}
}

\twocolumn[{
\renewcommand\twocolumn[1][]{#1}
\maketitle
\begin{center}
    \captionsetup{type=figure}
    \includegraphics[width=\textwidth]{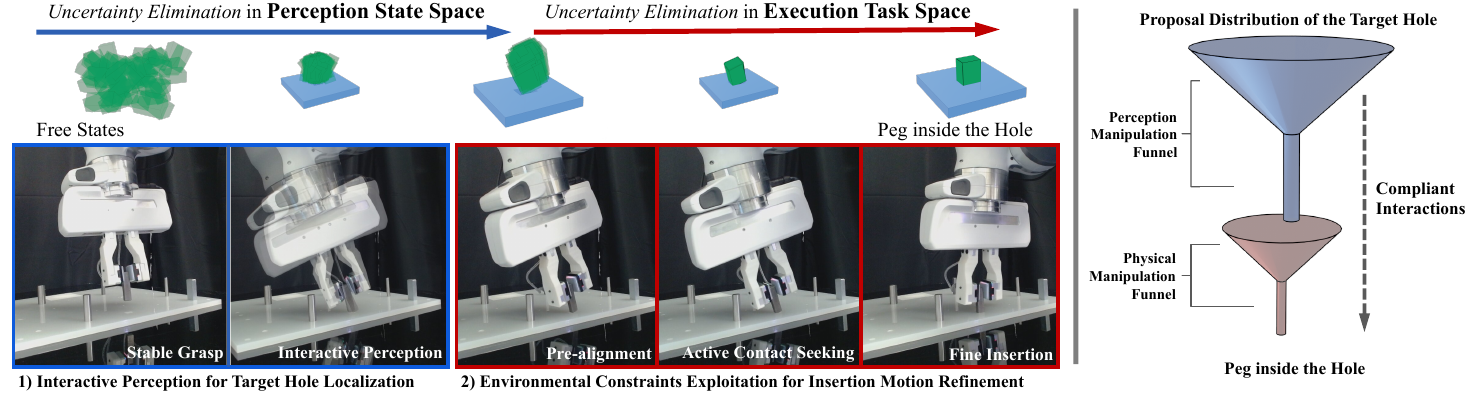}
    \captionof{figure}{\textbf{Motivation.} Acknowledging that real-world uncertainties are inevitable, we exploit environmental constraints to shape the manipulation process toward the desired outcome rather than expecting the robot to precisely execute any trajectories. The example demonstrates the system's two major components: 1) progressively identify environmental constraints to localize the target hole; 2) form a constant contact between the inclined peg and its corresponding corner from the hole to serve as motion constraints for insertion refinement. Such an uncertainty-absorbing paradigm is formulated based on the composition of \textbf{manipulation funnels} ---a concept that guarantees a strict space shrinkage through task-specific interactions, in the perception state space and execution task space. }
\label{overall}
\end{center}

}]

\begin{abstract}
Robust and adaptive robotic peg-in-hole assembly under tight tolerances is critical to various industrial applications. However, it remains an open challenge due to perceptual and physical uncertainties from contact-rich interactions that easily exceed the allowed clearance. In this paper, we study how to leverage contact between the peg and its matching hole to eliminate uncertainties in the assembly process under unstructured settings. By examining the role of compliance under contact constraints, we present a manipulation system that plans collision-inclusive interactions for the peg to 1) iteratively identify its task environment to localize the target hole and 2) exploit environmental contact constraints to refine insertion motions into the target hole without relying on precise perception, enabling a robust solution to peg-in-hole assembly. By conceptualizing the above process as the composition of funneling in different state spaces, we present a formal approach to constructing manipulation funnels as an uncertainty-absorbing paradigm for peg-in-hole assembly. The proposed system effectively generalizes across diverse peg-in-hole scenarios across varying scales, shapes, and materials in a learning-free manner. Extensive experiments on a NIST Assembly Task Board (ATB) and additional challenging scenarios validate its robustness in real-world applications.

\end{abstract}

\IEEEpeerreviewmaketitle

\section{Introduction}
Robotic peg-in-hole assembly is a foundational skill in various industrial applications and beyond, involving inserting a peg into its matching hole on another object under tight clearances~\cite{kimble2020benchmarking}. Common strategies seek to optimize the insertion trajectory under certain conditions, either through explicitly modeling the task environment or formulating the problem as a POMDP~\cite{lauri2022partially} (partially observable Markov decision process) to enable reactive policies based on sensory input. Regardless of the debate on how accurately a model describes reality, both types of approach neglect that the onus of fine manipulation lies in the execution~\cite{toussaint2014dual}. It is the \textit{execution} that brings the optimized result from planning into the real world through motor actuation by the controller. However, real-world perception challenges and physical uncertainties, such as those arising from hardware limitations, can introduce small deviations between planned and executed trajectories. Under the tight tolerance of peg-in-hole assembly, even minor uncertainties can lead to undesired manipulation outcomes. Therefore, robotic peg-in-hole assembly requires non-trivial effort to tackle real-world uncertainties beyond solely planned trajectories. 

Acknowledging that planned trajectories are rarely executed perfectly in the real world due to uncertainties, focusing on the outcomes of interactions between the system and the environment~\cite{eppner2015exploitation, eppner2015planning}, rather than on explicit trajectories, offers a promising approach to achieving robust manipulation. This formulation often requires the role of compliance, either from mechanical design~\cite{deimel2016novel, odhner2014compliant} or computational simulation~\cite{hogan1984impedance}, to passively adapt to external contacts. Rather than causing undesired outcomes, environmental constraints can shape the compliant manipulation process toward the desired outcome under deliberate exploitation~\cite{lozano1984automatic}, even without elaborate contact models~\cite{posa2014direct, cheng2022contact}. With this insight, we propose leveraging the contacts between the peg and its matching hole as an advantage to eliminate uncertainties for fine insertions.

In this paper, we introduce a manipulation system (shown in Fig.~\ref{overall}) that plans compliant interactions between the manipulated peg and its task environment to identify and exploit contact constraints for robust insertion. Safe contact is ensured by impedance control~\cite{hogan1984impedance}, which allows external contact to alter the pose of the peg by mimicking a spring-damper behavior. Since the steady pose of the peg after each interaction is observable and is a combined effect of the environmental constraints and potential from impedance control, the system can 1) identify its encountered constraints to shrink the target hole's proposal distribution progressively and 2) exploit the environmental constraints to shape the motion of the peg to enter its matching hole. Specifically, the exploitation process involves maintaining constant contact between the peg and the hole to serve as motion constraints. Since the environmental constraints can shape a large margin of uncertainties and the interaction mechanism of force-pose regulation is independent of the specific trajectory taken, the proposed system can tolerate uncertainties smaller than its execution precision, even under tight tolerances.

As the system dynamics progressively eliminates uncertainty in perception and execution, we formulate the proposed system as a composition of \textit{manipulation funnels} in different state spaces. The original concept of the \textit{manipulation funnel} was introduced by \citet{mason1985mechanics} to eliminate uncertainties in object location, emphasizing that the robustness of funnel-based manipulation requires the co-design of robot interaction and task mechanics. A general formulation of the funneling process can be defined in any task-relevant space as a progressive space shrinkage for uncertainty elimination. Inspired by the planar theory of automatic motion synthesis~\cite{lozano1984automatic}, we establish a formal approach to constructing manipulation funnels for robust fine insertion of general prismatic pegs in the real world. The primary contributions of this paper are:
\begin{itemize}
    \item We present an uncertainty-absorbing paradigm for general peg-in-hole insertion based on the concept of manipulation funnel, which does not plan on the explicit trajectories but rather the outcome of each compliant interaction. 
    \item  We present a formal approach to constructing the funneling process for both perception and execution, which can be further composed for general peg-in-hole insertion.
    \item  We provide a system with detailed algorithms to implement the above funnel-based manipulation paradigm on real-world assembly tasks.
    \item  We conduct ablation studies through extensive experiments on a standard assembly benchmark NIST ATB and additional challenging tasks to validate the robustness and generalizability of the proposed paradigm and system.
\end{itemize}

The rest of this paper is organized as follows: Section~\ref{related_work} reviews related works. Section~\ref{prelim} introduces the preliminaries and problem statement for funnel-based manipulation planning in peg-in-hole assembly. Section~\ref{funnel_based_planning} introduces the task mechanics and presents a formal approach to constructing manipulation funnels in perception and execution. Section~\ref{sec:exp} demonstrates the robustness of the proposed system in simulation and real-world experiments. We discuss the limitations under current settings in Section~\ref{sec:limits} and summarize the contributions in Section~\ref{sec:conclusion}.

\section{Related Work}

\label{related_work}

\subsection{Peg-in-Hole Assembly}
Over the past decade, pipeline-based and end-to-end approaches have driven advancements in robotic peg-in-hole systems. Pipeline-based approaches aim to explicitly estimate intermediate system states for motion planning. Previous works~\cite{jin2021contact, tang2016autonomous, kim2022active} leverage tactile or force/torque sensors for contact state estimation. Vision-based methods~\cite {xie2022learning, haugaard2021fast} estimate alignment deviations for correction or search motion planning. 
This modular design provides interpretability and flexibility, allowing for the integration of both analytical and data-driven methods. 
In contrast, end-to-end approaches~\cite{levine2016end} aim at deriving an implicit function to map sensory inputs to motor skills from trial-and-error. The entire system connects feature extraction to planning in a unified manner~\cite{dong2021tactile, luo2021robust, Tang-RSS-24, luo2018deep, inoue2017deep, luo2019reinforcement}. 

However, previous methods are mainly based on the principle that assembly is a relative positioning task~\cite{simunovia1979information}, assuming the robot will precisely execute the issued trajectory. Our research targets a different perspective~\cite{lozano1984automatic}, as the fine insertion motion is a passive refinement process that exploits environmental constraints.

\subsection{Compliance-enabled Manipulation}
Compliance mitigates unknown external disturbances through passive adaptability. The role of compliance in robot manipulation often involves compensation for uncertainty and safe interaction with the environment, which has been adopted in a wide range of tasks, including dexterous grasping~\cite{Chen-RSS-24, li2016dexterous, deimel2016novel}, in-hand manipulation~\cite{li2014learning, morgan2022complex, Bhatt-RSS-21, hang2021manipulation}, grasp adaptation~\cite{li2014learning_adap, hang2016hierarchical, khadivar2023adaptive}, assembly~\cite{morgan2021vision, luo2024serl, zhang2021learning} and human-robot co-manipulation~\cite{Shao-RSS-24, peternel2018robot}.

Additionally, compliance-enabled contact can safely exploit environmental constraints to eliminate uncertainties, ensuring robust and dexterous manipulation~\cite{eppner2015planning, hang2019pre, dafle2014extrinsic, shao2020learning, zhou2023learning, hou2020manipulation, chen2023sliding}. In this work, we propose a formal approach for compliance-enabled contact between the peg and the hole to eliminate uncertainty in insertion motions.

\subsection{Funnel-based Manipulation}
\citet{mason1985mechanics} proposed the original concept of \textit{manipulation funnels} to eliminate uncertainties based on task mechanics and task-specific interaction. The importance of exploiting task mechanics to eliminate uncertainties is also revealed in concepts such as \textit{pre-image}~\cite{lozano1984automatic} and \textit{backprojections}~\cite{erdmann1985using}. Early works have applied funnel-based designs in manipulation tasks, including ball batting~\cite{burridge1999sequential}, part orienting~\cite{goldberg1993orienting, erdmann1988exploration} and feeding~\cite{akella2000parts}, grasping~\cite{christiansen1991manipulation}, and peg-in-hole assembly~\cite{whitney1982quasi}. \citet{Bhatt-RSS-21} investigated funnel-based action primitives for in-hand manipulation through an empirical study. \citet{canberk2023cloth} proposes ``canonicalized-alignment" as task mechanics to funnel the large space of possible cloth configurations into a smaller, structured one.

The major challenges in developing funnel-based manipulation lie twofold: 1) formulating the general task mechanics for a given manipulation task and 2) defining task-specific interaction beyond hardware-associated control. To tackle the above challenges in real-world peg-in-hole tasks, we provide a general formulation of the task mechanics of peg-in-hole insertion and an object-centric interaction mechanism, which is not dependent on any specific hardware. 

\section{Preliminaries and Problem Statement}
\label{prelim}
In this work, we consider the peg-in-hole assembly problem as inserting a peg into its matching hole on a planar board surface as illustrated in Fig.~\ref{fig:task}-(a). We aim to plan compliant motions of a manipulated peg that frequently makes and breaks contact with its task environment, first to perceive the location of the matching hole on the task board and then insert the peg into it under tight clearance. 
\begin{figure}
    \centering
    \includegraphics[width=\linewidth]{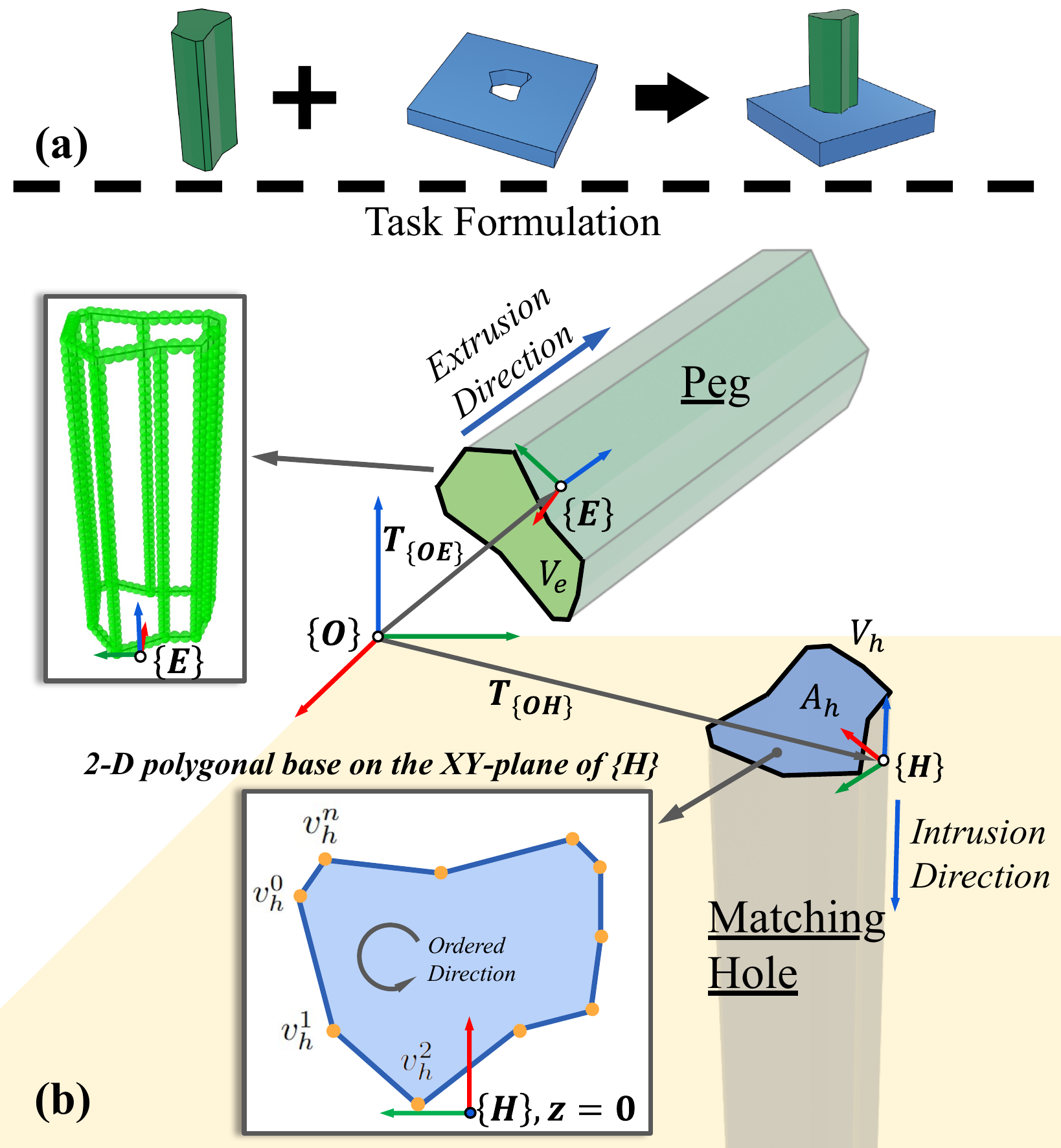}
    \caption{(a) The peg-in-hole problem is considered as inserting a peg into its matching hole on a planar board (a randomly generated peg is adopted as the example). (b) An overview of the task formulation is presented, including the geometry representation of the peg, the hole, and the overall task environment.}
    \label{fig:task}
\vspace{-0.5cm}
\end{figure}
 
\subsection{Preliminaries}

\subsubsection{Geometry of the Peg and Hole}\label{geo_pah} The 3-dimensional (3-D) geometry of the hole is described in the \textit{hole reference frame} $\{H\}$, as an intrusion of a 2-D polygonal base defined by an ordered sequence of vertices $\mathcal{V}_h=(v^0_h, v^1_h, ..., v^n_h), v_h\in\mathbb{R}^3, z_{v_h}=0$ on the $XY$-plane of  $\{H\}$. $S(\mathcal{V}_h)=\cup_{i=0}^n \{v_h^i \}\subset \mathbb{R}^3$ is the corresponding unordered set of the $\mathcal{V}_h$. The intrusion is along the negative direction of the $Z$-axis of $\{H\}$. $\mathcal{V}_h$ defines a planar region $\mathcal{A}_h\subset\mathbb{R}^3$ on the $XY$-plane of $\{H\}$ as:
\begin{equation}
\label{area_def}
    \mathcal{A}_h = \{ (x, y, z)\in \mathbb{R}^3  \mid f_{\mathcal{A}_h}(x,y)\leq0 \land z=0\}
\end{equation}
where $f_{\mathcal{A}_h}(x, y)$ is an implicit distance function to describe the boundary of $\mathcal{A}_h$ as $f_{\mathcal{A}_h}(x, y)=0$ and its interior area as $f_{\mathcal{A}_h}(x, y)<0$, exterior area as $f_{\mathcal{A}_h}(x, y)>0$.

The 3-D geometry of the peg is described in the \textit{peg reference frame} $\{E\}$, as an extrusion of a 2-D polygon base  $\mathcal{V}_e=(v^0_e, v^1_e, ..., v^n_e), v_e\in\mathbb{R}^3, z_{v_e}=0$ on the $XY$-plane of  $\{E\}$. Its corresponding unordered set is denoted as $S(\mathcal{V}_e)=\cup_{i=0}^n \{v_e^i \}\subset \mathbb{R}^3$.The extrusion is along the positive direction of $Z$-axis of $\{E\}$, and $\mathcal{V}_e$ is considered approximated equal to $\mathcal{V}_h$ but defined in different frames. 

\subsubsection{Task Environment}
We specify the \textit{Cartesian world frame} $\{O \}$ attached to the task board, as its $XY$-plane is aligned with the board surface and its orthogonal $Z$-axis points outward the task board. The pose of the hole in the world frame is denoted as $T_{\{OH\}}\in \text{SE(3)}$. We describe the task environment using a occupancy function $\Gamma: \mathbb{R}^3\mapsto [-1, 0, 1]$, which is defined as follows:
\small
\begin{equation}
\label{task_env}
\Gamma(x, y, z) = \left\{
             \begin{array}{lr}
             -1,&  z<0 \land f_{T_{\{OH\}}\mathcal{A}_{h}}(x, y) >0 \\
            1, &  z>0  \lor (z\leq0 \land  f_{T_{\{OH\}}\mathcal{A}_{h}}(x, y) < 0) \\
            0, & \text{Otherwise} \\
\end{array}
\right.
\end{equation}
\normalsize

in which $-1$ denotes occupied by the task board, $0$ denotes on the surface of the task board, and $1$ denotes the free space. An initial area of interest $\mathcal{A}_\text{board}\subset\mathbb{R}^3$ is specified on the $XY$-plane of $\{O\}$ that guarantees $T_{\{OH\}}\mathcal{A}_h\subset \mathcal{A}_\text{board}$. 

 \subsubsection{Compliant Interactions for the Peg}
As the peg is in a prism-shaped geometry as defined in Sec.~\ref{geo_pah}, we use the densely sampled points $\mathcal{P}_e=\{p_e^i\in \mathbb{R}^3\}_{i=1}^M$ on the edges of the peg for collision checking (considered continuous along the lateral and base edges), in which $p_e=(x,y,z)$ is in the Cartesian coordinates of $\{E\}$. The sampled points are assumed to be rigidly attached to each other, as the entire point set undergoes the same rigid body transformation (rotation and translation). The pose of frame $\{E\}$  with reference to frame $\{O\}$ at time $t$ is represented as $\mathbf{x}_{t}=[\mathbf{p}; \mathbf{r}]\in\mathbb{R}^6$, in which $\mathbf{p}\in \mathbb{R}^3$ is the position vector and $\mathbf{r}\in \mathbb{R}^3$ is the orientation represented with Euler angle. The target peg-in-hole state is denoted as $\mathbf{x}^*$. The corresponding transformation matrix of $\mathbf{x}_{t}$ is denoted as $T_{\{OE\},t}\in \text{SE(3)}$. An overview of the task environment with the manipulated peg is illustrated in Fig.~\ref{fig:task}-(b).

To enable safe contact between the peg and its task environment, we formulate the impedance control based on the desired state $\mathbf{x}^{\text{d}}_t\in\mathbb{R}^6$ at time $t$. A wrench $\mathbf{F}\in \mathbb{R}^6$ (consisting of 3-D force and torque) is applied to the peg by a Cartesian impedance controller as formulated:
\begin{equation}
\label{impedance}
    \mathbf{F} = \mathbf{K}_d (\mathbf{x}^{\text{d}}_t - \mathbf{x}_{t}) - \mathbf{D}_d \dot{\mathbf{x}}_{t}
\end{equation}
in which $\mathbf{K}_d\in\mathbb{R}^{6\times6}$ is the diagonal stiffness matrix to decouple compliance in each degree of freedom, $\mathbf{D}_d\in \mathbb{R}^{6\times6}$ is the damping matrix for movement stabilizing and $\dot{\mathbf{x}}_{t}\in \mathbb{R}^6$ is the velocity of $\{E\}$ with reference to $\{O\}$ at time $t$.

An interaction command $\mathbf{c}_t =(\mathbf{x}_{t}, \mathbf{x}_{t}^{\text{d}})$ at time $t$ is defined by its starting state $\mathbf{x}_{t}$ (considered steady as $\dot{\mathbf{x}}_{t}=0$) and a desired state $\mathbf{x}_{t}^{\text{d}}$. A compliant interaction $\mathbf{x}_{t+1}=\Pi(\mathbf{x}_{t}, \mathbf{x}_{t}^\text{d})$ between the peg and its task environment is an unknown but observable process $\Pi: \mathbb{R}^6\times\mathbb{R}^6\mapsto\mathbb{R}^6$ from the starting state $\mathbf{x}_{t}$ to the next steady state $\mathbf{x}_{t+1}$, under the drive of $\mathbf{x}_{t}^\text{d}$ and the environmental constraints in the frame of $\{O\}$. Except for the virtually defined desired state $\mathbf{x}_{t}^\text{d}$, any physical existing state $\mathbf{x}_{t+1}$ during this process is constrained by its task environment as follows:
\begin{equation}
\label{non-penetration}
    \forall p_o \in T_{\{OE\},t+1}\mathcal{P}_e, \Gamma(x_{p_o},y_{p_o},z_{p_o} )\neq-1
\end{equation}
We do not explicitly predict or trace the trajectory of the peg during this process but only observe its steady state at $\mathbf{x}_{t}$ and $\mathbf{x}_{t+1}$. 
At the steady state $\mathbf{x}_{t+1}$, any external contact wrench $\mathbf{F}_\text{ext}\in \mathbb{R}^6$ in frame $\{O\}$ is balanced by the pose deviation as:
\begin{equation}
\label{ext_force}
    \mathbf{F}_\text{ext} = \mathbf{K}_d (\mathbf{x}^{\text{d}}_t - \mathbf{x}_{t+1})
\end{equation}
while the velocity $\dot{\mathbf{x}}_{t+1}=0$ of the peg remains zero.

We use the intersection point set $\mathcal{P}_{o, \text{footprint},t}\subset \mathbb{R}^3$ between the peg's edges and the $XY$-plane of $\{O \}$ at the steady pose $\mathbf{x}_{t}$ to describe the outcome of each interaction as:
\begin{equation}
   \mathcal{P}_{o, \text{footprint}} = \{ p_o \in T_{\{OE\},t}\mathcal{P}_e\mid \forall p_o=(x,y,z) , z=0 \}
\end{equation}
$\mathcal{P}_{o, \text{footprint}}$ reflects the spatial relation between the peg and the task board.

\subsection{Problem Statement}
\label{pb_state}
We formulate the peg-in-hole assembly task as a progressive uncertainty elimination process in the state space of perception and execution.
We aim to plan interactions $\mathbf{c}_{t}=(\mathbf{x}_{t}, \mathbf{x}_{t}^{\text{d}})$ between the peg and its task environment to 1) iteratively reduce the uncertainty of the perception result of the target hole and 2) finish the peg insertion under execution and perception uncertainty. An overview of the peg-in-hole assembly process is outlined in Alg.~\ref{alg:two_stage}.

\textbf{Perception State Space:} We use $\widetilde{T}_{\{OH\}}$ to denote the estimated state of the ground truth $T_{\{OH\}}$. The estimation result at time $t$ is represented with a probabilistic density function $P_t(\widetilde{T}_{\{OH\}})$ bounded by a proposal region $\mathcal{X}_t\subset \text{SE(3)}$, which is defined as follows:
\begin{equation}
\label{prob_vf}
P_t(\widetilde{T}_{\{OH\}}) = \left\{
             \begin{array}{lr}
             \frac{1}{|\mathcal{X}_t|}, & \widetilde{T}_{\{OH\}} \in \mathcal{X}_t \\
            0, & \widetilde{T}_{\{OH\}} \notin \mathcal{X}_t
\end{array}
\right.
\end{equation}
where $|\mathcal{X}_t|$ denotes the volume of region $\mathcal{X}_t$. 

We aim to form a recursive sequence of proposal regions based on the outcomes $\mathcal{P}_{o, \text{footprint},t}$ of each interaction $\mathbf{c}_t$:
\begin{equation}
\label{virtual_funnel}
\mathcal{X}_0 \supseteq \mathcal{X}_1 \supseteq \mathcal{X}_2 \supseteq \dots \supseteq \mathcal{X}_{T1} 
\end{equation}
As $\mathcal{X}_{t}$ shrinks over steps, the expected spread of $\widetilde{T}_{\{OH\}}$ decreases and the uncertainty range of the perceived hole's state is reduced.

\textbf{Execution Task Space:} Let $\Delta \mathbf{x}\in\mathbb{R}$ be the deviation between the steady state $\mathbf{x}_{t}$ and the peg-in-hole state $\mathbf{x}^*$. Based on the estimated state distribution of $P_{T1}(\widetilde{T}_{OH})$, we aim to shrink $\Delta \mathbf{x}$ at each step of interaction $\mathbf{c}_t$ to progressively reach the peg-in-hole configuration as follows:
\small
\begin{equation}
\label{physical_f}
    \Delta \mathbf{x}_{0}  > \Delta \mathbf{x}_{1} > \Delta \mathbf{x}_{2} > \dots > \Delta \mathbf{x}_{T2}
\end{equation}
\normalsize
We consider a successful peg-in-hole assembly when $\Delta \mathbf{x}$ reaches zero.

\begin{algorithm}
\caption{Peg-in-Hole Assembly}
\renewcommand{\algorithmicrequire}{\textbf{Input:}}
\renewcommand{\algorithmicensure}{\textbf{Output:}}
\begin{algorithmic}[1]
\label{alg:two_stage}
\REQUIRE base polygon $\mathcal{V}_{h}$, initial area of interest $\mathcal{A}_\text{board}$, peg geometry as $\mathcal{P}_e$, world frame $\{O\}$
\STATE $t\gets 0$ \COMMENT{Initialize time step for perception}
\STATE $\mathcal{X}_{T1}\gets$ InteractivePerception($\mathcal{V}_{h}, \mathcal{A}_\text{board}, \mathcal{P}_e, \{O\}$) \COMMENT{Alg.~\ref{interactive_p}}
\STATE $t\gets 0$ \COMMENT{Reset time step for execution}

\STATE $\mathbf{x}_{t} \gets$ CornerAlignment($\mathcal{X}_{T1}$) \COMMENT{Alg.~\ref{alg:alignment}}
\WHILE{not reach $\mathbf{x}^*$}   
\STATE $\mathbf{x}_{t}^{\text{d}} \gets$ InsertionPlanning($\mathbf{x}_{t}$) \COMMENT{Alg.~\ref{alg:insertion}}
\STATE $\mathbf{x}_{t+1}\gets \mathbf{c}_t(\mathbf{x}_t, \mathbf{x}_{t}^{\text{d}})$
\STATE $t\gets t+1$
\ENDWHILE

\end{algorithmic}
\end{algorithm}

\section{Funnel-based Manipulation Planning}
\label{funnel_based_planning}

A manipulation funnel represents a convergence process by iterative contractions from a larger entrance to a smaller exit state space, ensured by task-specific interactions. A general funnel formulation is defined as \textit{Definition}~\ref{funnel_form}.
\begin{definition}
\label{funnel_form}
Let $S$ be a task-relevant space that could represent any system parameters. Given a target state $\mathbf{s}^*$, we aim to find subspaces $S_{\text{in}} \subset S$ so that after applying some state transition functions, denoted as $\Pi: S\rightarrow S$, the resultant output space $S_{\text{out}} \subset S$ is guaranteed to be strictly smaller than the input space: $|S_{\text{in}}|> |S_{\text{out}}|$. If these conditions are satisfied , $\mathbf{s}^*\in S_\text{in}$ and $\mathbf{s}^*\in S_\text{out}$, a tuple $\mathcal{F} = (S_\text{in}, \Pi, S_\text{out})$ is called a manipulation funnel.

\end{definition}

We first define the task-specific interactions based on the task mechanics in Section~\ref{interaction}. Then,  we introduce the formal approach to construct manipulation funnels in perception state space (Section~\ref{sec:vf}) and execution task space (Section~\ref{sec:pf}).

\subsection{Task Mechanics and Interaction Primitives}
\label{interaction}
Funnel-based manipulation planning requires deep exploitation of the intrinsic task mechanics~\cite{mason1985mechanics}. As perception and physical uncertainties are inevitable~\cite{rodriguez2021unstable}, our key insight is forming an aligned \textit{corner} between the inclined peg and the target hole to create contact constraints for undesired motion freedoms and progressively enter the allowed clearance (as illustrated in Fig.~\ref{fig:mechanics}-(a)). A \textit{corner} from the hole is defined by arbitrary vertex $v_h^\text{corner}\in S(\mathcal{V}_h)$ as the local angle $\angle v_h^\text{corner}$ formed by $v_h^\text{corner}$ and its nearby edges from the interior side. In this paper, we use the case of $\angle v_h^\text{corner} < \pi$ for motion funnel construction. However, theoretically, it also applies to the case of $\angle v_h^\text{corner} > \pi$ as long as the local convexity exists. 

\begin{figure}[H]
    \centering
    \includegraphics[width=\linewidth]{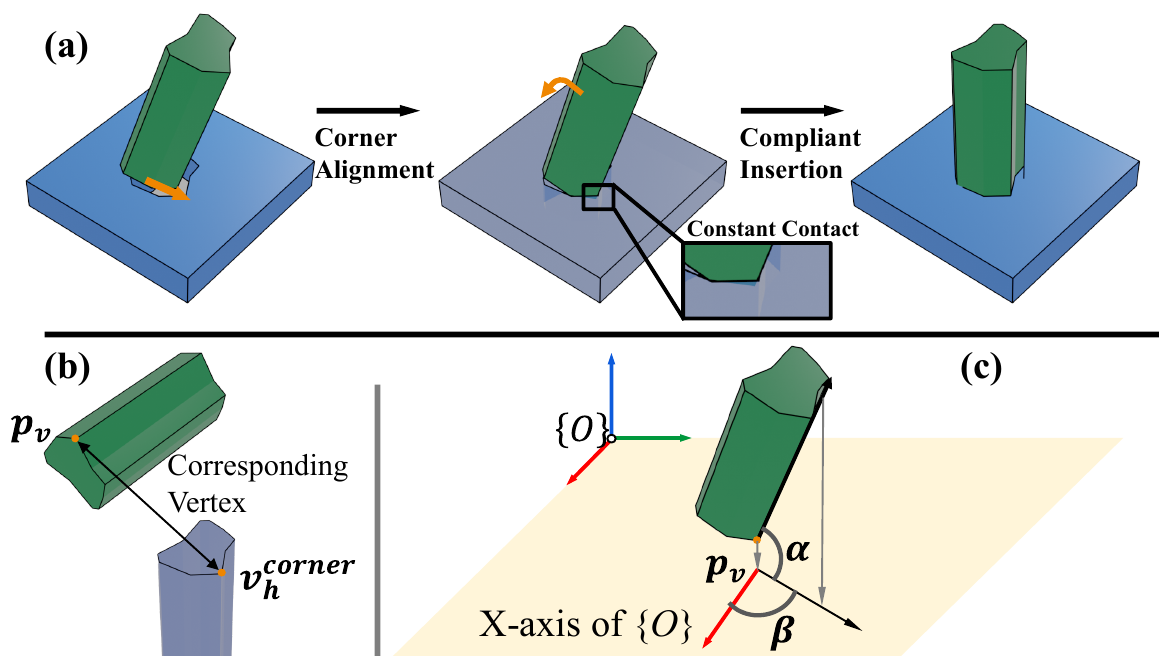}
    \caption{(a) The task mechanics of peg-in-hole insertion: first, constant contact between the peg and the hole is formed; second, the formed constraints are leveraged to refine the insertion motions. (b) A paired corner on the peg and hole: this local geometry enables the downstream iterative insertion process. (c) The inclined state is constrained by its supporting vertex $p_v$, inclined angle $\alpha$ and rotation angle $\beta$.}
    \label{fig:mechanics}
\end{figure}

Despite the trajectory being a dominant action representation in manipulation planning, it is unsuitable for funnel-based manipulations as interactions with the task environment are allowed to alter the motion of the manipulator \cite{mason1985mechanics}.  Given the target corner $v_{h}^\text{corner}$ of the hole, its corresponding vertex on the base of the peg is denoted as $p_v=v_{e}^\text{corner}\in S(\mathcal{V}_e)\subset \mathcal{P}_e$, as shown in Fig.~\ref{fig:mechanics}-(b). To integrate the alignment mechanism as an inductive bias for interactions, the task-specific interaction $\mathbf{c}_t=(\mathbf{x}^{d}_t, \mathbf{x}_t)$ is extended as $\mathbf{c}^v_t=(\mathbf{x}^{d}_t, \mathbf{x}_t, p_v)$ in \textit{Definition}~\ref{inclined_state}.

\begin{definition}
\label{inclined_state}
For a task-specific interaction $\mathbf{c}_t^v$, the starting state $\mathbf{x}_t$ and desired state $\mathbf{x}_t^{d}$ is constrained by a common supporting vertex $p_v$ defined as follows: 
\begin{equation}
\begin{aligned}
\label{pv_vertex}
       \forall \mathbf{x}\in [\mathbf{x}^{d}_t, \mathbf{x}_t]:  \forall p_o, p_o\in T_{\{OE\}}\mathcal{P}_e \land p_o\neq T_{\{OE\}}p_v \\, z_{p_o} > z_{T_{\{OE\}}p_v} 
\end{aligned}
\end{equation}  
Such a formulation ensures that the vertex $p_v$ on the peg base possesses the lowest value along the $Z$-axis of $\{O\}$, while the interaction process is unchanged as $\mathbf{x}_{t+1}=\Pi(\mathbf{x}_{t}, \mathbf{x}_{t}^\text{d})$.
\end{definition}

An inclined state is visualized in Fig.~\ref{fig:mechanics}-(c). Let $\mathbf{n}$ represents the positive direction of $Z$-axis of $\{E\}$ with reference to $\{O\}$ at the inclined state, we use $\alpha$ as the inclined angle and $\beta$ as the rotation angle to define $\mathbf{n}$ as follows:
\begin{equation}
\label{n_vertex}
    \mathbf{n} = \begin{bmatrix}
\cos(\alpha) \cos(\beta) \\
\cos(\alpha) \sin(\beta) \\
\sin(\alpha)
\end{bmatrix}
\end{equation}

Interaction with inclined states is designed to identify and exploit its environmental contact constraints. 
Such a formulation offers several advantages: 1) when the peg engages with the task board from above (as illustrated in Fig.~\ref{fig:env_id}), the observation $\mathcal{P}_{o, \text{footprint},t}$ is divided into binary categories based on the cardinality of the set $|\mathcal{P}_{o, \text{footprint},t}|$: whether it is a contact point or an intersection area between the peg and hole; 2) an inclined state relaxes the requirement for a peg to be partially entered into the hole~\cite{lozano1984automatic}; 3) before an insertion task is finished (considered vertical w.r.t. the board surface plane), all peg states can be described as inclined states during the insertion process.

\begin{figure}
    \centering
    \includegraphics[width=\linewidth]{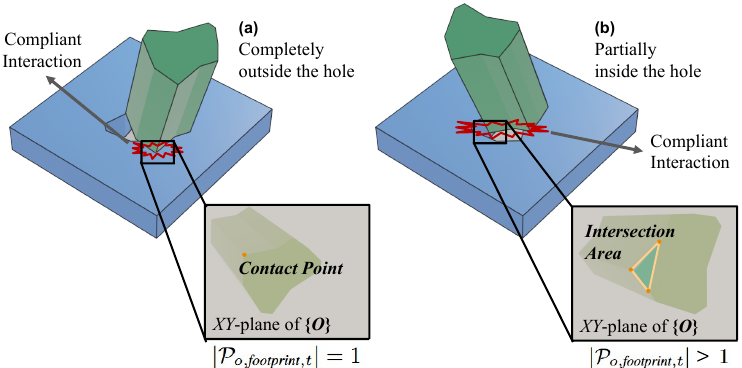}
    \caption{The observation $\mathcal{P}_{o,\text{footprint},t}$ from $\mathbf{c}_t^v$ can be divided into binary categories, (a) contact point when $|\mathcal{P}_{o,\text{footprint},t}|=1$ or (b) intersection area when $|\mathcal{P}_{o,\text{footprint},t}|>1$, based on its cardinality.}
    \label{fig:env_id}
\vspace{-0.4cm}
\end{figure}

\subsection{Perception Manipulation Funnel}
\label{sec:vf}
The perception manipulation funnel aims to iteratively reduce the proposal region $\mathcal{X}_t$ of the target hole's distribution $P_t(\widetilde{T}_{\{OH\}})$ through observation from each interaction. By strategically placing the supporting vertex $p_v$ of the desired state $\mathbf{x}_t^{\text{d}}$ beneath the $XY$-plane of $\{O\}$, we can intentionally capture the encountered contact constraints reflected on the resulted steady state $\mathbf{x}_{t+1}$. As illustrated in Fig.~\ref{fig:env_id}, after each interaction, the observation $\mathcal{P}_{o, \text{footprint},t}$ can be interpreted as an inequality constraint as defined in \textit{Definition}~\ref{ineq} to limit the distribution of $\widetilde{T}_{\{OH\}}$.

\begin{definition}
\label{ineq}
An inequality constraint on the possible space of $\widetilde{T}_{\{OH\}}$ is defined as $g_t(\widetilde{T}_{\{OH\}}) \geq 0$ based on the observation $\mathcal{P}_{o, \text{footprint},t}$, where $g_t(\widetilde{T}_{\{OH\}})$ is defined as:
\begin{equation}
\begin{aligned}
 \forall p_o\in& \mathcal{P}_{o, \text{footprint},t}, \\
 &g_t(\widetilde{T}_{\{OH\}}) = \left\{
             \begin{array}{lr}
            f_{\widetilde{T}_{\{OH\}}\mathcal{A}_h}(p_o), |\mathcal{P}_{o,\text{footprint},t}|=1& \\
           -\max f_{\widetilde{T}_{\{OH\}}\mathcal{A}_h}(p_o), \text{Otherwise} &
\end{array}
\right.
\end{aligned}
\end{equation}
in which $|\cdot|$ is the cardinality of the set. For a potential hole's state $\widetilde{T}_{\{OH\}}$, any contact point (as $ |\mathcal{P}_{o,\text{footprint},t}|=1$) should be excluded from the hole's area, while any intersection area (as $ |\mathcal{P}_{o,\text{footprint},t}|>1$) should be fully contained by the hole's area.
\end{definition}

The proposal region at $t=0$ is derived from the initial area of interest $\mathcal{A}_{\text{board}}$ as follows:
\begin{equation}
\label{eq:init_x0}
    \mathcal{X}_0 = \{ \widetilde{T}_{\{OH\}}\mid \widetilde{T}_{\{OH\}}\mathcal{A}_h \subset \mathcal{A}_\text{board}   \}
\end{equation}

The proposal region at step \( t \) is defined as:
\begin{equation}
\label{region_t}
\mathcal{X}_t = \{\widetilde{T}_{\{OH\}} \in \mathcal{X}_0\subset SE(3) \mid  g_i(\widetilde{T}_{\{OH\}}) \geq 0, \, i=1, 2,..., t\}    
\end{equation}  
which is updated at step \( t+1\) with the new added constraint $g_{t+1}(\widetilde{T}_{\{OH\}}) \geq 0$ as:
\begin{equation}
\label{eq:region_update}
\mathcal{X}_{t+1} = \mathcal{X}_{t} \cap \{\widetilde{T}_{\{OH\}} \in SE(3) \mid g_{t+1}(\widetilde{T}_{\{OH\}}) \geq 0 \}
\end{equation}

\begin{lemma}
\label{fs_region}
As defined in Eq.~\eqref{eq:region_update}, the volume $|\mathcal{X}_t|$ of the feasible region $\mathcal{X}_t$ is monotonically decreasing ($\mathcal{X}_{t+1}\subseteq \mathcal{X}_{t}$) over nonidentical interactions $\mathbf{c}_t^v$.
\end{lemma}
\begin{proof}
Since $\mathcal{X}_{t+1}$ is the intersection of $\mathcal{X}_{t}$ with another constraint set $g_{t+1}(\widetilde{T}_{\{OH\}})\geq 0$ and $\{g_{t}(\widetilde{T}_{\{OH\}})\geq 0\} \neq \{g_{t+1}(\widetilde{T}_{\{OH\}})\geq 0\}$ under nonidentical interactions,  $|\mathcal{X}_t|$ is decreasing or stays the same when new constraints are added.
\end{proof}
As $\mathcal{X}_t$ shrinks over steps in \textit{Lemma}~\ref{fs_region}, the expected spread of $\widetilde{T}_{\{OH\}}$ decreases, and the allowable state space of the target hole narrows.

\begin{lemma}
\label{vf_approach}
The ground truth state of the hole is always included in the allowable state space as $T_{\{OH\}}\in \mathcal{X}_t$ over interactions $\mathbf{c}_t^v$.
\end{lemma}
\begin{proof}
As $T_{\{OH\}} \in \mathcal{X}_0$ and $T_{\{OH\}}\in \{\widetilde{T}_{\{OH\}}\mid g_t(\widetilde{T}_{\{OH\}})\geq0 \}$, thus proven $T_{\{OH\}} \in \mathcal{X}_t$ as defined in Eq.~\eqref{region_t}.
\end{proof}

Since the proposal region $\mathcal{X}_{t+1}\subseteq \mathcal{X}_{t}$ of the target hole is monotonically shrinking over interactions $\mathbf{c}_t^v$, meanwhile the target state $T_{\{OH\}}$ is guaranteed to be contained by $\mathcal{X}_t$ and $\mathcal{X}_{t+1}$, we introduce the perception manipulation funnel as defined in \textit{Theorem}~\ref{pf_conv}.
\begin{theorem}
\label{pf_conv}
The process of proposal region shrinkage $\mathcal{X}_{t+1}\subseteq \mathcal{X}_{t}$ through compliant interactions $\mathbf{c}_t^v$ is a manipulation funnel in the perception state space.
\end{theorem}
\begin{proof}
As proven in \textit{Lemma}~\ref{fs_region} and~\ref{vf_approach}, the proposal region $\mathcal{X}_t$ is shrinking and approaching $T_{\{OH\}}$ over random interactions $\mathbf{c}_t^v$
, thus proving that the perception manipulation funnel is established based on the general \textit{Definition}~\ref{funnel_form}
\end{proof}

Additionally, a maximum entropy-based method is introduced to improve convergence efficiency. An overview is outlined in Alg.~\ref{interactive_p}. The probability distribution $P_t(\widetilde{T}_{\{OH\}})$ of the target hole's state is defined by Eq.~\ref{prob_vf} and the initial area for exploration $\mathcal{A}_\text{board}$ is divided into grid points $\mathcal{G}=\{\mathbf{g}_i \mid i=1, 2, ..., N\}$ where $\mathbf{g}_i\in \mathbb{R}^3$ on the $XY$-plane of $\{O\}$. During each interaction, number $K$ of potential hole's states is sampled as $\{\widetilde{T}_{\{OH\}, i}\}^K_{i=1}\sim P_t(\widetilde{T}_{\{OH\}})$. The probability $P_{\text{in}}$ that $\mathbf{g}$ is included by the target hole is defined by:

\begin{gather}
    P_{\text{in}}(\mathbf{g}) = \frac{\sum^K_{i=1}h(\widetilde{T}_{\{OH\}, i}, \mathbf{g})}{K},\\
    h(\widetilde{T}_{\{OH\}}, \mathbf{g}) =   \left\{
             \begin{array}{lr}
            1, & \mathbf{g}\in \widetilde{T}_{\{OH\}}\mathcal{A}_{h} \\
           0, & \text{Otherwise}
\end{array}
\right. 
\end{gather}
Let $ \mathcal{H}(\cdot)$ represent the entropy of the given probabilistic density function. The expected entropy reduction is formulated as:
\begin{equation}
    \Delta \mathcal{H}(\mathbf{g}) = \mathcal{H}(P_t(\widetilde{T}_{\{OH\}})) - E[\mathcal{H}(P_t(\widetilde{T}_{\{OH\}})\mid \mathbf{g})]
\end{equation}
where $E[\mathcal{H}(P_t(\widetilde{T}_{\{OH\}})\mid \mathbf{g})]$ is the expected entropy after evaluating $\mathbf{g}$.
To maximize information gain, the desired state should spatially overlap with the selected grid point $\mathbf{g}^* \in \mathcal{G}$ that maximizes $\Delta \mathcal{H}(\mathbf{g})$ for contact constraints identification:
\begin{equation}
\label{eq:grid_max}
\mathbf{g}^* = \arg\max_{\mathbf{g} \in \mathcal{G}} \Delta \mathcal{H}(\mathbf{g})
\end{equation}
The convergence of the proposed entropy-based exploration is proven in \textit{Theorem}~\ref{pf_conv}. 

\begin{algorithm}
\caption{Entropy-based Interactive Perception}
\renewcommand{\algorithmicrequire}{\textbf{Input:}}
\renewcommand{\algorithmicensure}{\textbf{Output:}}
\label{interactive_p}
\begin{algorithmic}[1]
\REQUIRE  $\mathcal{V}_{h}$, $\mathcal{A}_\text{board}$, $\mathcal{P}_e$, $\{O\}$
\STATE $t\gets0$
\STATE Sample grid points $\mathcal{G}\gets\mathcal{A}_\text{board}$
\STATE Initialize proposal region $\mathcal{X}_t\gets\mathcal{A}_\text{board}$ \COMMENT{Eq.\eqref{eq:init_x0}}
\WHILE{not converged}  
    \STATE  $\mathbf{g}^*\gets$ Select grid point for interaction  \COMMENT{Eq.\eqref{eq:grid_max}}
    \STATE  $\mathcal{P}_{o,\text{footprint},t}$ $\gets$ Active contact with the selected point
    \STATE  $\mathcal{X}_{t+1}\gets$ ConstraintsUpdate($\mathcal{P}_{o,\text{footprint},t}$)  \COMMENT{Eq.\eqref{eq:region_update}}
    \STATE $t\gets t+1$
\ENDWHILE
\ENSURE bounded region $\mathcal{X}_{T1}$

\end{algorithmic}
\end{algorithm}

\subsection{Physical Manipulation Funnel}
\label{sec:pf}

\begin{figure*}
    \centering
    \includegraphics[width=0.9\textwidth]{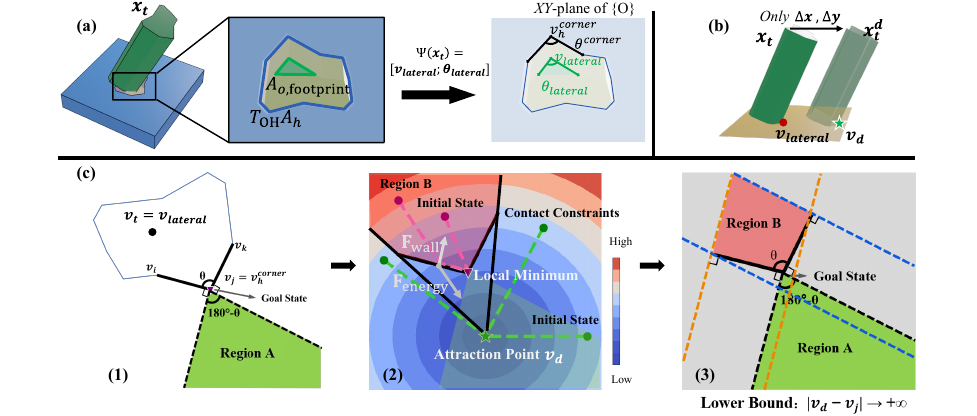}
    \caption{(a) The peg's state $\mathbf{x}_t$ is projected as  $[v_{\text{lateral}}; \theta_\text{lateral}]$. (b) By setting the translation deviation between the desired state $\mathbf{x}^{\text{d}}_t$ and the starting state $\mathbf{x}_t$ only on the $XY$-plane of $\{O\}$, a 2-D potential energy field is formulated. (c) The corner alignment process is formulated as aligning the point $v_t = v_\text{lateral}$ to the corner $v_j=v_h^\text{corner}$ under the drive of the created potential energy field, while the edges are considered non-penetration walls.}
\label{subspace}
\vspace{-0.5cm}
\end{figure*}

The physical manipulation funnel aims to leverage environmental contacts as physical constraints in the execution task space to iteratively reduce $\Delta \mathbf{x}$ for peg-in-hole insertion. Based on the task mechanics as Fig.~\ref{fig:mechanics}-(a), we consider a sequential process as 1) aligning an inclined state to its corresponding corner and 2) adjusting the aligned inclined state to insertion. All the states $\mathbf{x}_t$ described in this section are considered inclined states, which are defined by $p_v$, $\alpha$, and $\beta$ as introduced in Eq.~\eqref{pv_vertex} and~\eqref{n_vertex}. We describe the spatial relationship between the $\mathbf{x}_t$ and $\mathbf{x}^*$ (the target peg-in-hole state) based on the $\mathcal{P}_{o, \text{footprint},t}$ and the target corner $v_h^\text{corner}$. Specifically,
$\mathcal{P}_{o, \text{footprint},t}$ forms a 2-D polygon as $\mathcal{V}_{o, \text{footprint},t}=(p_o^0, p_o^1,...,p_o^k), p_o\in\mathcal{P}_{o, \text{footprint},t}$ with a corresponding area $\mathcal{A}_{o, \text{footprint},t}\subset\mathbb{R}^3$ on the $XY$-plane of $\{O\}$. The state deviation $\Delta \mathbf{x}$ is defined based on $\mathcal{V}_{o, \text{footprint},t}$ and $v_h^\text{corner}$ as \textit{Definition}~\ref{subspace_proj}.

\begin{definition}
\label{subspace_proj}
Given an inclined peg state $\mathbf{x}_t$ with the lateral edge of the supporting vertex $p_v$ intersects with the $XY$-plane of $\{O\}$, we project its state as $\Psi: \mathbb{R}^6\mapsto \mathbb{R}^4$. We use the $v_\text{lateral}\in \mathbb{R}^3$ and the angle value $\theta_\text{lateral}\in\mathbb{R}$ of $\angle v_\text{lateral}$ to describe the pose as $\Psi(\mathbf{x}_t)=[v_{\text{lateral}}; \theta_{\text{lateral}}]$, as illustrated in Fig.~\ref{subspace}-(a). $v_{\text{lateral}}$ is the intersection point between the lateral edge of $p_v$ and the $XY$-plane of $\{O\}$, $\angle v_\text{lateral}$ is the angle formed by $v_{\text{lateral}}$ and its nearby vertices in $\mathcal{V}_{o, \text{footprint},t}$. Given the target corner for alignment $v_h^\text{corner}$ and its corresponding angle $\angle v_h^\text{corner}$, the deviation $\Delta \mathbf{x}$ is transformed into $\Delta \mathbf{s}$ as follows:
\begin{equation}
    \Delta \mathbf{s} = \overset{\text{translation deviation:}\Delta v}{\left\|v_{\text{lateral}}-T_{\{OH\}}v_h^{\text{corner}} \right\|} +  \overset{\text{inclined deviation:}\Delta \theta}{|\theta^\text{corner}-\theta_\text{lateral}|}
\end{equation}
where $\theta^\text{corner}, \theta_\text{lateral}$ is the angle value of $\angle v_h^\text{corner}$ and $\angle v_\text{lateral}$ respectively. $\Delta \mathbf{x}\rightarrow0$ is considered equal to $\Delta \mathbf{s}\rightarrow 0$.
\end{definition}

Since the footprint fully contains the projection of the intersection part beneath the board surface plane, the environmental constraints as defined in Eq.~\ref{non-penetration} are projected into the $XY$-plane of $\{O\}$ as $\mathcal{A}_{o, \text{footprint},t}\subset T_{\{OH\}}\mathcal{A}_h$. 

In the following, we describe the process of inclined corner alignment as $\Delta v$ reaches zero and the process from inclined to insertion as $\Delta \theta$ reaches zero. Since each degree of freedom in the Cartesian space of $\{O\}$ is decoupled as $\mathbf{K}_d$ is a diagonal stiffness matrix, we aim to leverage the potential caused by translation deviation for corner alignment and rotational deviation for insertion based on \textit{Lemma}~\ref{energy_f_l}. 

\begin{lemma}
\label{energy_f_l}
Under the quasi-static assumption, the motion of the peg $\mathbf{x}$ is considered driven by a potential energy field:
\begin{equation}
\label{energy_peg}
U(\mathbf{x}) = \frac{1}{2} (\mathbf{x}_d - \mathbf{x})^\top \mathbf{K}_d (\mathbf{x}_d - \mathbf{x})   
\end{equation}
Given an interaction $\mathbf{c}_t^v=(\mathbf{x}_t, \mathbf{x}^{\text{d}}_t)$, the resulted steady pose $\mathbf{x}_{t+1}$ automatically falls into its nearby local minimum energy under the non-penetration environmental constraints as $U(\mathbf{x}_t)\geq U(\mathbf{x}_{t+1}) \geq U(\mathbf{x}_t^{\text{d}})$ .
\end{lemma}
\begin{proof}
The peg tends to rest at a lower energy state in the potential energy field. Also, the peg cannot move to a higher energy state without external force.
\end{proof}

\textbf{Inclined Corner Alignment:} By constraining the deviation between the starting state $\mathbf{x}_t$ and desired state $\mathbf{x}_t^{\text{d}}$ only as planar translation in the $XY$-plane of $\{O\}$ (as illustrated in Fig.~\ref{subspace}-(b)), the applied external force to the peg defined by Eq.~\eqref{ext_force} can be projected into the $XY$-plane as follows:
\begin{equation}
    \mathbf{F}_\text{ext}^\text{planar} = \mathbf{K}_d [\Delta x, \Delta y, 0, 0, 0, 0]^\text{T}
\end{equation}
The formulation of $\mathbf{F}_\text{ext}^\text{planar}$ can be rewritten into a potential energy field based on the desired point $v_d\in\mathbb{R}^3$ and any point $v\in\mathbb{R}^3$ on the $XY$-plane of $\{O\}$ in a quadratic form:
\begin{equation}
\label{energy}
U_\text{XOY}(v) = \frac{1}{2} (v_d - v)^\top \mathbf{K}_d|_{1:3, 1:3} (v_d - v)    
\end{equation}
As illustrated in Fig.~\ref{subspace}-(c12), for an angle formed by the target corner vertex $v_j = T_{\{OH\}}v_h^\text{corner}$ and its nearby vertices, its edges are considered non-penetration walls in $U_\text{XOY}(v)$. Let $v_t$ be the position of $v_\text{lateral}$ at time $t$; we aim to leverage such constraints to shape the motion of $v_t$ to rest at  $v_j$ under the drive of $U_\text{XOY}(v)$.
Specifically, our objective is to formulate a \textit{potential well} to let $v_j$ be the local minimum in a potential energy field so that $v_t$ tends to rest at $v_j$ without escaping. This process involves two fundamental requirements: 1) finding an area $\mathcal{A}_\text{well}$ to place the attraction point $v_d$ (defined by $\Psi(\mathbf{x}_t^\text{d})$) and 2) finding the \textit{basin of attraction} $\mathcal{A}_\text{basin}$ where all initial states $v_t$ inside it will eventually move to $v_j$. Both $\mathcal{A}_\text{well}$ and $\mathcal{A}_\text{basin}$ should only condition on the non-penetration constraints from angle $\angle v_i v_j v_k$.

We start with a straightforward example of aligning a point $v_t$ to a determinate vertex $v_j$ in the small interior area of an angle $\angle v_i v_j v_k$ on the $XY$-plane of $\{O\}$. First, the formulation of $\mathcal{A}_\text{well}$ is defined in \textit{Lemma}~\ref{well1}. Then, we introduce the formulation of $\mathcal{A}_\text{basin}$ in \textit{Lemma}~\ref{well2}.

\begin{lemma}
\label{well1}
Area $\mathcal{A}_{\text{well}} \subset \mathbb{R}^3$ exist on the $XY$-plane of $\{O\}$ and only condition on the angle $\angle v_i v_j v_k$ (as visualized in Fig.~\ref{subspace}-(c1): Region A), for any desired point $v_d\in\mathcal{A}_{\text{well}}$, $v_j$ is always the local minimum state in the created basin of attraction $\mathcal{A}_{\text{basin}}^{v_d}\subset \mathbb{R}^3$ (as visualized in Fig.~\ref{subspace}-(c2): Region B). $\mathcal{A}_{\text{well}}$ is defined as:
\begin{equation} 
\label{wells}
    \mathcal{A}_{\text{well}} = \{v\in \mathbb{R}^3 \mid \frac{\pi}{2}\leq \angle v_i v_j v \leq (\frac{3\pi}{2} - \angle v_i v_j v_k) \}
\end{equation}
while $\mathcal{A}_{\text{basin}}^{v_d}$ is defined as:
\small
\begin{equation}
    \mathcal{A}_{\text{basin}}^{v_d} = \{v\in \mathbb{R}^3 \mid 
    \angle v_i v_d v \leq \angle v_i v_d v_k \cap \angle v_i v_j v \leq \angle v_i v_jv_k\}    
\end{equation}
\normalsize

\end{lemma}
\begin{proof} 
The absolute distance between $v_d\in\mathcal{A}_\text{well}$ and $v_t$ represents the energy level of a quadratic-formed energy field. As $\forall v_d\in \mathcal{A}_\text{well}, \forall v_t \in \mathcal{A}_{\text{basin}}^{v_d}, v_t\neq v_j: |v_d-v_t|>|v_d-v_j|$, thus proven $v_j$ is the local minimum state in $\mathcal{A}_{\text{basin}}^{v_d}$, $\mathcal{A}_\text{well}$ and $\mathcal{A}_{\text{basin}}^{v_d}$ establish.
\end{proof} 

\begin{lemma}
\label{well2}
Area $\mathcal{A}_{\text{basin}}\subset \mathbb{R}^3$ exist on the $XY$-plane of $\{O\}$ and only condition on the angle $\angle v_i v_j v_k$ (as visualized in Fig.~\ref{subspace}-(c3): Region B). For any desired point $v_d\in\mathcal{A}_{\text{well}}$ defined in Eq.~\eqref{wells}, $v_j$ is always the local minimum state in $\mathcal{A}_{\text{basin}}$. 
$\mathcal{A}_{\text{basin}} = \cap \{\mathcal{A}_{\text{basin}}^{v_d}\mid \forall v_d\in\mathcal{A}_{\text{well}}\}$ is defined as 
\begin{equation}
\begin{aligned}
 \mathcal{A}_{\text{basin}} =  & \{v\in \mathbb{R}^3 \mid \angle v v_i v_j \leq \frac{\pi}{2} \cap \angle v v_k v_j \leq \frac{\pi}{2}\}\\ &\text{when } \mid v_d - v_j \mid
  \rightarrow\infty \land v_d\in \mathcal{A}_{\text{well}}
\end{aligned}
\end{equation}
\end{lemma}
\begin{proof}
The original energy gradient of $U_\text{XOY}(v)$ is derived as $\mathbf{F}_\text{energy}= -\nabla U_\text{XOY}(v)$. While $\mid v_d - v_j \mid \neq 0$, any $v_t\in\mathcal{A}_\text{basin}$ moves along with the force gradient and will get in contact with the non-penetration wall first. As long as $v_t$ is in contact with the wall, the component of the energy gradient $\mathbf{F}_\text{energy}$ that is normal to the wall is canceled out by the non-penetration constraint as illustrated in Fig.~\ref{subspace}-(b2). Given the wall's normal vector $\mathbf{n}_\text{wall}$ at the contact point, the modified force gradient is derived as $\mathbf{F}_\text{total}=\mathbf{F}_\text{energy}+\mathbf{F}_\text{wall}=-\nabla  U_\text{XOY}(v) + (\nabla  U_\text{XOY}(v)\cdot\mathbf{n}_\text{wall})\mathbf{n}_\text{wall}$. For any $v_t$ in contact with the wall, it performs physical gradient descent under force $\mathbf{F}_\text{total}$ and reaches the local minimum as the force gradient $\mathbf{F}_\text{total}=0$ when $\mathbf{n}_\text{wall}$ is on the same line of $\mathbf{F}_\text{energy}$. As proven in \textit{Lemma}~\ref{well1}, the local minimum of energy field is at $v_j$, thus proven $\mathcal{A}_{\text{basin}}$ establish.
\end{proof}

Based on $\mathcal{A}_\text{well}$ and $\mathcal{A}_\text{basin}$, the corner alignment process is introduced in \textit{Lemma}~\ref{ali} through potential well construction.  
\begin{lemma}
\label{ali}
Under the condition that energy caused by translation deviation is significantly larger than that caused by rotation deviation, by setting the desired state $\mathbf{x}_t^{\text{d}}$ with its projected $v_d\in\mathcal{A}_{\text{well}}$, any starting state $\mathbf{x}_t$ with $v_t\in \mathcal{A}_\text{basin}$ will automatically align with the target corner (as the lateral edge of $p_v$ intersects with $ v_j$ at state $\mathbf{x}_{t+1}=\Pi(\mathbf{x}_t, \mathbf{x}_{t}^{\text{d}})$).
\end{lemma}
\begin{proof}
As the command from impedance control and the target corner creates a potential well with $v_j$ as its local minimum energy state, as proved in \textit{Lemma}~\ref{well1} and~\ref{well2}, $\mathbf{x}_{t+1}$ automatically rests at the local minimum state as $v_{t+1}\rightarrow v_j$.
\end{proof}

We relax the assumption on a determinate $v_j = T_{\{OH\}}v_h^\text{corner}$ and expand the potential well construction process on the probabilistic distribution of the target hole's state $P(\widetilde{T}_{\{OH\}})$, as outlined in Alg.~\ref{alg:alignment}. $K$ states are sampled as $\{\widetilde{T}_{\{OH\},i}\}^{K}_{i=1} \sim P(\widetilde{T}_{\{OH\}})$ and the ground truth of $T_{\{OH\}}$ is assumed to be covered the samples. By calculating the intersection area of $\widetilde{\mathcal{A}}_\text{basin}=\cap \{\mathcal{A}_{\text{basin,i}}\}^k_{i=1}$ and $\widetilde{\mathcal{A}}_\text{well}=\cap \{\mathcal{A}_{\text{well,i}}\}^K_{i=1}$ over pose samples $\{\widetilde{T}_{\{OH\},i}\}^{K}_{i=1}$, it is guaranteed to find a sub-area as $\widetilde{\mathcal{A}}_\text{well}\subset T_{\{OH\}}\mathcal{A}_\text{well}$ and $\widetilde{\mathcal{A}}_\text{basin}\subset T_{\{OH\}}\mathcal{A}_\text{basin}$ for potential well construction. Theoretically, the robustness of the corner alignment is conditioned on  $\widetilde{\mathcal{A}}_\text{well}$ and $\widetilde{\mathcal{A}}_\text{basin}$ instead of the geometric size of $\mathcal{V}_{h}$. 
\begin{algorithm}
\caption{Nondeterministic Corner Alignment}
\renewcommand{\algorithmicrequire}{\textbf{Input:}}
\renewcommand{\algorithmicensure}{\textbf{Output:}}
\begin{algorithmic}[1]
\label{alg:alignment}
\REQUIRE bounded region $\mathcal{X}_{T1}$, $P(\widetilde{T}_{\{OH\}})\gets \mathcal{X}_{T1}$
\STATE Initialize samples $\{\widetilde{T}_{\{OH\},i}\}^{K}_{i=1} \sim P(\widetilde{T}_{\{OH\}})$, $T_{\{OH\}} \in \widetilde{T}_{\{OH\},i}\}^{K}_{i=1}$, $\widetilde{\mathcal{A}}_\text{well} = \widetilde{\mathcal{A}}_\text{basin} = {XY}\text{-plane of }  \{O\}$
\FOR{$i = 0,1,..., K-1$}
\STATE $\widetilde{\mathcal{A}}_\text{well}$ = $\widetilde{\mathcal{A}}_\text{well}\cap \widetilde{T}_{\{OH\},i}\mathcal{A}_\text{well}$
\STATE $\widetilde{\mathcal{A}}_\text{basin}$ = $\widetilde{\mathcal{A}}_\text{basin}\cap \widetilde{T}_{\{OH\},i}\mathcal{A}_\text{basin}$
\ENDFOR
 \STATE $v_\text{lateral}\gets \Psi(\mathbf{x})$, $v_\text{lateral}\in \widetilde{\mathcal{A}}_\text{basin}$
\STATE $v^{\text{d}}_\text{lateral}\gets \Psi(\mathbf{x}^{\text{d}})$, $v^{\text{d}}_\text{lateral}\in \widetilde{\mathcal{A}}_\text{well}$
\STATE Execute interaction $c^v=(\mathbf{x}, \mathbf{x}^{\text{d}})$
\end{algorithmic}
\end{algorithm}

\begin{figure*}
    \centering
    \includegraphics[width=\textwidth]{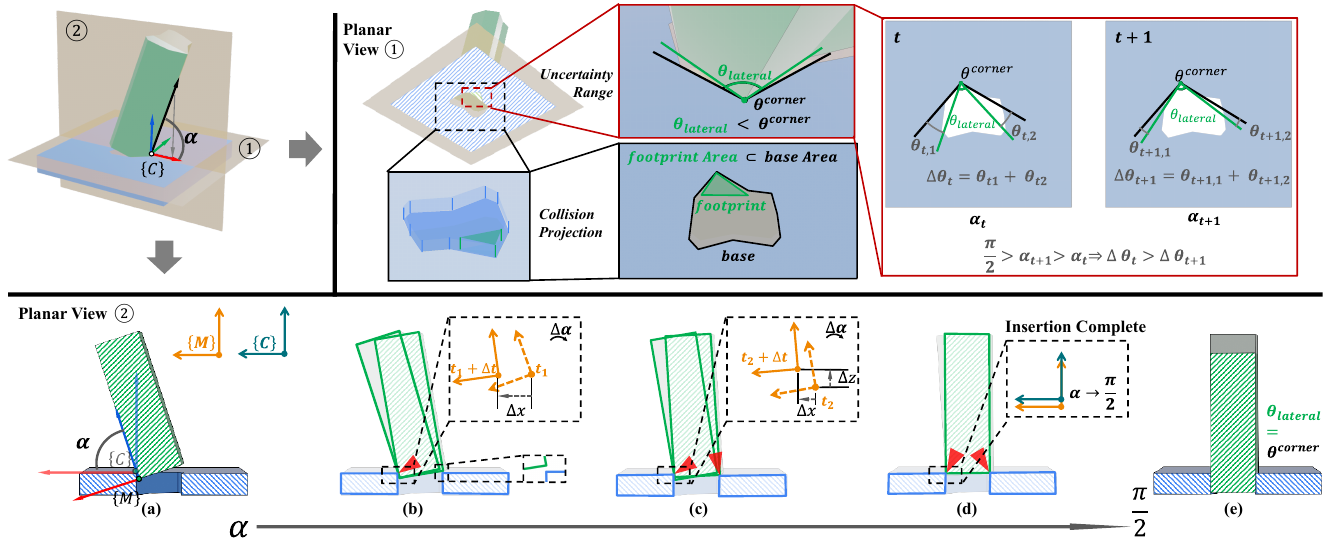}
    \caption{An overview of the motion constraints based on the $XY$-plane (planar view 1) and $XZ$-plane (planar view 2) of $\{C\}$. \textbf{Planar View 1:} An uncertainty w.r.t to rotation is constrained by the range of $\Delta\theta$ through interactions. Meanwhile, $\Delta\theta$ is also decreasing significantly over insertion motions. \textbf{Planar View 2:} The translation of $\{M\}$ w.r.t $\{C\}$ induced by the potential energy field caused by rotation deviation and environmental constraints (red friction cones) is a passive result to balance $\mathbf{F}_{\text{ext}}$, thus the position of the origin of $\{M\}$ will not reach the $XY$-plane of $\{C\}$ as the process finishes at $\alpha\rightarrow\frac{\pi}{2}$.}
    \label{fig:physical}
\vspace{-0.4cm}
\end{figure*}
 \textbf{Insertion under Motion Constraints:} 
As constant contact between the peg and its matching hole is established, we aim to leverage the motion constraints from the environment to synthesize insertion motions. Successful insertion motions are formulated as a sequence of interactions $S = [\mathbf{c}_0^v, \mathbf{c}_1^v, ..., \mathbf{c}_t^v]$ that connect the initial inclined state to the target peg-in-hole configuration. We keep the alignment tight by keeping the projection $v_d$ of the desired state inside $v_d\in \widetilde{\mathcal{A}}_{well}$. The alignment is formed by the energy field caused by the translation deviation; additionally, we aim to leverage the energy field caused by the rotation deviation to drive the insertion process.

\begin{definition}
\label{insertion_p}
Let $\alpha^*=\frac{\pi}{2}$, the insertion process is formulated as follows: under the condition that the corner contact is formed as a pivot point, a complete insertion is regarded as progressively turning the inclined angle of the peg approaching vertical as $\alpha \rightarrow \alpha^*$. 
\end{definition}

We describe the relative motion between the peg and its matching hole using a \textit{contact reference frame} and a \textit{manipulation frame}. As illustrated in Fig.~\ref{fig:physical}, a \textit{contact reference frame} $\{C\}$ is attached to the on the corner $T_{\{OH}\}v_{h}^{\text{corner}}$, where its origin is at $T_{\{OH\}}v_{h}^{\text{corner}}$ and its $Z$-axis shares the same direction with $\{O\}$. The positive direction of its $X$-axis is aligned with the rotation angle $\beta$ from the inclined state of the peg on the $XY$-plane of $\{O\}$. The $Y$-axis is placed orthogonal with reference to the $XZ$-plane of $\{C\}$. By definition, $\{C\}$ is not fixed with reference to ${\{O\}}$ as its $X$-axis is defined by the rotation angle $\beta$ of the peg's inclined state. 

\begin{lemma}
During the insertion process (\textit{Definition}~\ref{insertion_p}), $\{C\}$ is a constrained frame with a fixed position and a limited orientation range with reference to $\{O\}$.
\end{lemma}
\begin{proof} 
As the lateral edge of $p_v$ intersects with the corner vertex as illustrated in the \textit{Planar View 1} from Fig.~\ref{fig:physical}, the possible pose of the peg is restricted by its projection $\theta_{\text{lateral}}<\theta^{\text{corner}}$ and the uncertainty range is subject to $\Delta\theta$. Since the orientation of the inclined peg is subject to $\Delta\theta$, proving that $\{C\}$ is a constrained frame with reference to $\{O\}$.
\end{proof}

Selecting a \textit{manipulation frame} $\{M\}$ attached to the peg is critical to describe its relative motion in $\{C\}$. Since the interactions $\mathbf{c}_t^v$ are constrained as \textit{Definition}~\ref{inclined_state} by a supporting vertex $p_v$, we set the origin of $\{M\}$ to $p_v$ with its $Z$-axis's positive direction towards $\mathbf{n}$. The $X$-axis of $\{M\}$ is placed orthogonally in the plane formed by the lateral edge of $p_v$ and the $X$-axis of $\{C\}$ as illustrated in the \textit{Planar View 2} of Fig.~\ref{fig:physical}. We denote the instantaneous relative motion of $\{M\}$ with reference to $\{C\}$ at time $t$ as a twist vector $\xi_t= [
\bm{\omega}\quad \mathbf{d}]^\text{T}$, in which $\bm{\omega}=[\omega_x, \omega_y,\omega_z]\in\mathbb{R}^3$ is the rotation rate along each axis of $\{C\}$ and $\mathbf{d}=[d_x,d_y,d_z]\in\mathbb{R}^3$ is the rate of change of the position.  

\begin{lemma}
During the insertion process (\textit{Definition}~\ref{insertion_p}), $\xi_t$ is restricted to translation in the $XZ$-plane of $\{C\}$ and rotation about the $Y$-axis of $\{C\}$, forming a subset of SE(3).
\end{lemma}
\begin{proof} 
Since the $XZ$-plane of $\{C\}$ and the $XZ$-plane of $\{M\}$ lies in the same geometric plane in $\{O\}$, thus proving that the relative motion is restricted as $\xi = [0, \omega_y, 0, d_x, 0, d_z]$
\end{proof} 
Since frame $\{C\}$ is within a limited range, it can be regarded as a static frame at any instant movement $\xi_t$ described in \textit{Lemma}~\ref{rot}.
\begin{lemma}
\label{rot}
Under the condition that the energy caused by rotation deviation is significantly larger than that caused by translation deviation, by setting the desired state's orientation as a small rotation $\omega_y\Delta t$ on the current state, the deviation $\Delta\alpha = \alpha^*-\alpha $ is decreasing, and the uncertainty range $\Delta\theta$ of $\{M\}$ is shrinking.
\end{lemma}
\begin{proof} As illustrated in the \textit{Planar View 2} of Fig.~\ref{fig:physical}, the translation motion $\mathbf{v}$ is a result refined by the passive force from the environmental contacts. The peg cannot break the alignment according to \textit{Lemma}~\ref{energy_f_l}, as the result $\{M\}$ is always lower than $\{C\}$ in the world frame.  Such formulation also guarantees that an enlarged inclined angle represents a smaller potential energy state in the caused energy field; the transition between the peg's state automatically leads to a higher inclined angle under its environmental motion constraints according to \textit{Lemma}~\ref{energy_f_l}.  As illustrated in the \textit{Planar View 1} of Fig.~\ref{fig:physical}, the rotation motion $\omega_x$ and $\omega_z$ is constrained by the limited range of $\Delta\theta$, as the $\alpha$ increases, $\Delta\theta$ decreases. 
\end{proof} 

By incrementally increasing the inclined angle $\alpha$ and updating the measurement of the inclined state's rotation angle $\beta$, the peg's state will progressively move towards the target state while its uncertainty range $\Delta\theta$ is shrinking towards zero. Theoretically, the robustness of the insertion process is conditioned on the peg's state $\textbf{x}_t$ instead of its geometric size.
\begin{theorem}
\label{pf_insert}
The insertion process introduced by inclined angle adjustment after corner alignment is a manipulation funnel in the execution task space based on the general \textit{Definition}~\ref{funnel_form}.
\end{theorem}
\begin{proof}
The deviation in the peg's state $\Delta\mathbf{s}=\Delta v + \Delta\theta$ is decreasing decoupled. As $\Delta v$ decreases to zero as proven in \textit{Lemma}~\ref{ali} and $\Delta\theta$ decreases to zero as proven in \textit{Lemma}~\ref{rot} over environmentally constrained interactions $\mathbf{c}_t^v$, thus proving the physical manipulation funnel established.
\end{proof}

\begin{algorithm}
\caption{MPC-based Insertion Planning}
\renewcommand{\algorithmicrequire}{\textbf{Input:}}
\renewcommand{\algorithmicensure}{\textbf{Output:}}
\begin{algorithmic}[1]
\label{alg:insertion}
\REQUIRE system dynamics model $\widetilde{\Pi}$, initial state $\mathbf{x}_t$, target inclined angle $\alpha^*$
\WHILE{not $|\alpha_t - \alpha^*|\rightarrow 0$}
\STATE $(\mathbf{x}_t^{\text{d}}, ..., \mathbf{x}_{t+T-1}^{\text{d}})\gets$ ConstrainedOptimization \COMMENT{Eq.\eqref{mpc_min}}
\STATE Execute the interaction $\mathbf{c}_t^v=(\mathbf{x}_t, \mathbf{x}_t^{\text{d}})$
\STATE Re-measure $\mathbf{x}_{t+1}$, $\alpha_{t+1}\gets \mathbf{x}_{t+1}$
\STATE Update system dynamics model $\widetilde{\Pi}$  \COMMENT{\hyperref[app:trans]{Appendix}}
\STATE $t\gets t+1$ 
\ENDWHILE
\end{algorithmic}
\end{algorithm}

Additionally, we introduce a Model Predictive Control (MPC) framework to find the formulated insertions process $S$. We initialize the system transition function $\widetilde{\Pi}$ with a linear prior and update it online with a Recursive Least Squares (RLS) adaptive filter (see \hyperref[app:trans]{Appendix}). At time step $t$, the system sets a desired state $\mathbf{x}_t^{\text{d}}$, and the state of the peg transitions from $\mathbf{x}_t$ to $\mathbf{x}_{t+1}$ under environmental constraints. Interactions during insertion are successive, which means the resulting steady state from the previous interaction serves as the starting state for the next interaction. The cost function is defined by the inclined angle $\alpha_t$ at state $\mathbf{x}_t$ and the target inclined angle $\alpha^*$ at $\mathbf{x}^*$ as $\mathcal{J}(\mathbf{x}_t, \mathbf{x}^*)= |\alpha_t - \alpha^*|$. A sequence of actions $(\mathbf{x}_t^{\text{d}}, ..., \mathbf{x}_{t+T-1}^{\text{d}})$ is selected to minimize the defined loss over a finite horizon $T$ as:
\small
\begin{equation}
\label{mpc_min}
\begin{aligned}
(\mathbf{x}_t^{\text{d}}, ...,\mathbf{x}_{t+T-1}^{\text{d}}) & = \arg\min \sum_{i=0}^{T-1} (\mathcal{J}(\Pi(\mathbf{x}_{t+i}, \mathbf{x}_{t+i}^{\text{d}}), \mathbf{x}^*) + \mathbf{u}_{t+i})\\
\text{subject to: }  &\Delta\mathbf{x}_\text{min} \leq  |\mathbf{x}_{t+i} - \mathbf{x}_{t+i}^{\text{d}}| \leq \Delta\mathbf{x}_\text{max}  \\
& \Psi(\mathbf{x}_{t+i}^{\text{d}})=[v_\text{lateral}, \theta_\text{lateral}], v_\text{lateral}\in \mathcal{A}_\text{well}\\ 
& i = 0, \dots, N-1
\end{aligned}
\end{equation}
\normalsize
where $\mathbf{u}_{t+i} = ||\mathbf{x}_{t+i} - \mathbf{x}_{t+i}^{\text{d}}||$ is added to ensure smooth control inputs by discouraging unnecessarily large or aggressive control actions. At each step $t$, the system applies the first selected desired state to execute and updates the transition model $\widetilde{\Pi}$. An overview of the MPC framework is outlined in Alg.~\ref{alg:insertion}, and its convergence is proved in \textit{Lemma}~\ref{rot}.

\begin{figure*}[h]
    \centering
    \includegraphics[width=\textwidth]{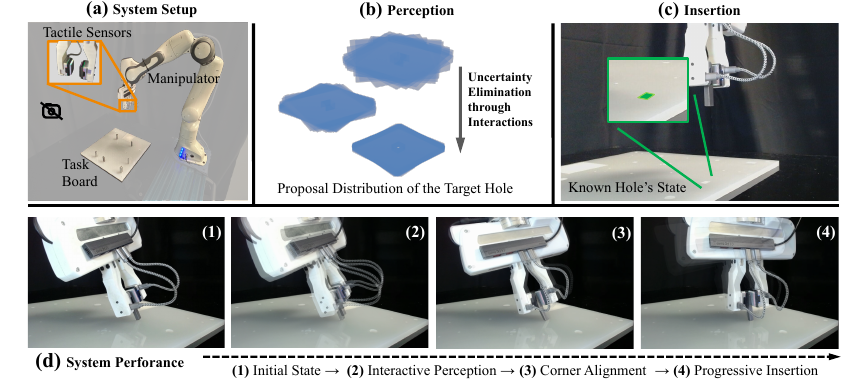}
    \caption{(a) Overview of the System Setup; (b) Ablation study on the perception manipulation funnel; (c) Ablation study on the physical manipulation funnel;  (d) Performance evaluation on the overall funnel-based manipulation system.}
    \label{fig:system}
\end{figure*}

\section{Experiments} 
\label{sec:exp}
In this section, we evaluate the proposed framework across a diverse range of peg-in-hole tasks to demonstrate its robustness and generalizability. Ablation studies are conducted to gain deeper insights into how the funnel-based manipulation approach mitigates uncertainties arising from perception and execution. Specifically, these experiments aim to validate: (1) the perception manipulation funnel eliminates uncertainties in the proposal distribution of the target hole's state while tolerating uncertainties from the imprecise interactions; (2) the physical manipulation funnel absorbs perception and execution uncertainties to achieve precise insertion; (3) by leveraging funnel-based manipulation planning, the system achieves robust and precise insertion tasks with a tolerance level surpassing the robot’s inherent execution precision.

\subsection{Experiment Setup}

Our experimental setup is shown in Fig.~\ref{fig:system}-(a). We test the proposed funnel-based manipulation framework on a robotic system consisting of a Franka Emika Panda manipulator with GelSight Mini attached to the fingertips of its parallel jaw gripper. No camera is required as our system eliminates uncertainties purely through physical interactions. This design removes any system dependence on variable lighting conditions and textures, making it easily generalizable to different scenarios. It is worth noting that, though visual modality is not mandatory, it can be easily integrated as providing the initial estimation of the target hole, which can be further refined by interactive perception or directly bridged with insertion. The in-hand pose of the grasped peg is observed from its corresponding tactile imprints, as discussed in the \hyperref[app:inhand]{Appendix}. Our system design does not require slipping-free conditions since the proposed funnel theory is object-centric. Two types of controllers are adopted in our real-world experiments: 1) a Cartesian impedance controller with a translation error up to $2mm$ and 2) a Cartesian position controller with a translation error up to $1mm$ due to imprecise internal models. A detailed formulation of Cartesian impedance control for a manipulator is discussed in the \hyperref[app:imp]{Appendix}. The Cartesian position controller is applied in position-based commands, and the Cartesian impedance controller is applied in commands requiring physical interactions. Under external contacts, the object undergoes only minor elastic movements on the tactile sensor’s planar surface, while major pose deviations are regulated by the robot's Cartesian impedance controller through the end-effector.

The system is tested on a NIST ATB benchmark and additional tight-clearance tasks, with detailed parameters of the peg-in-hole tasks shown in Fig.~\ref{fig:nist_atb}. The experimental pegs are selected with different base symmetries (ranging from central symmetry and axial symmetry to asymmetry), materials, scales, and clearance levels for a comprehensive evaluation. It is also worth noting that, although \textit{chamfer}\footnote{In the context of peg-in-hole tasks, the chamfer refers to the beveled or tapered edge at the entrance of the hole (and sometimes at the tip of the peg). A chamfer added to the edges can help reduce contact forces, reduce jamming and tolerate misalignment.} on the peg and the hole would make the insertion process easier, as it relatively enlarges the entry space, it is excluded in our experiments for a more challenging task formulation.

\begin{figure}[htbp]
  \centering
  \includegraphics[width=0.9\linewidth]{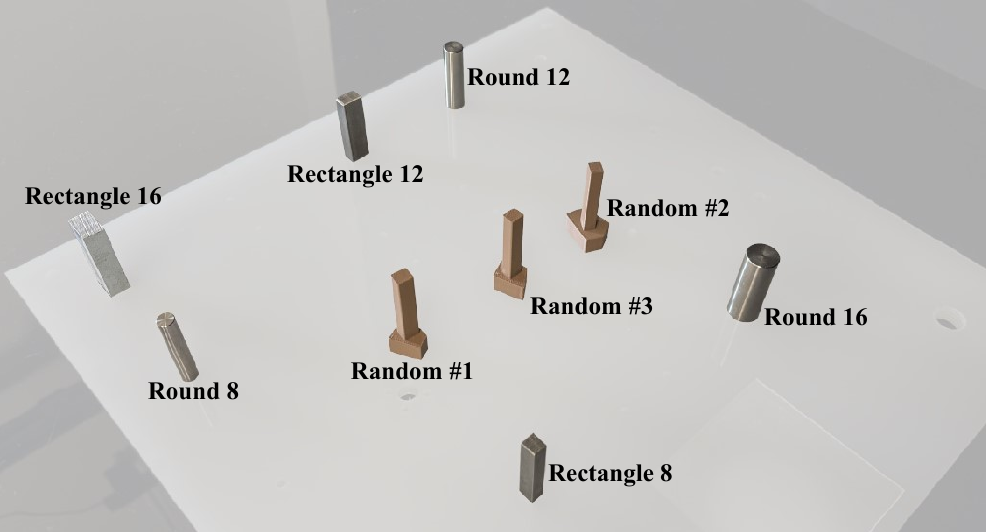} 
  \vspace{10pt} 

\scriptsize
\begin{tabular}{ |p{1.6cm}||p{1.cm}<{\centering}|p{1.2cm}<{\centering}|p{1.2cm}<{\centering}|p{1.3cm}<{\centering}|  }
 \hline
 Name & Scale$^*$ & Clearance$^*$ & Material &  Symmetry\\
 \hline
 Round 8    &  $\Phi 8$    & $\sim 0.8$  &   metal & central\\
 Round 12 &  $\Phi 12$  &  $\sim 0.8$    & metal& central\\
Round 16  & $\Phi 16$ &  $\sim 0.8$  & metal & central\\
Rectangle  8 & $8\times7$ & $\sim 0.6$  &  metal& axial\\
 Rectangle 12&   $12 \times 8$  & $\sim 0.7$ & metal& axial\\
 Rectangle 16 & $16 \times 10$  &  $\sim 0.8$    & metal& axial\\
 Random \#1$^{**}$ & $20\times16$ & $\sim 0.4$ & polylactide &  \textbf{asymmetric} \\
Random \#2$^{**}$ & $22\times 25$ & $\sim 0.4$ & polylactide &  \textbf{asymmetric} \\
 Random \#3$^{**}$ & $23\times17$ & $\sim 0.4$ & polylactide &  \textbf{asymmetric} \\

 \hline
\end{tabular}
 \begin{tablenotes}   
        \footnotesize              
        \item $^{*}$ Scale and clearance are denoted in millimeters.
        \item $^{**}$ Pegs with random generated polygonal base; scales are described with the bounding box of the base.

\end{tablenotes}   
  \caption{Overview of the Peg-in-Hole Tasks in Real-world Experiments.}
  \label{fig:nist_atb}
\vspace{-1.7cm}
\end{figure}

\subsection{Perception Manipulation Funnel}
\label{exp:vf}
We conducted detailed experiments in PyBullet~\cite{coumans2016pybullet} to evaluate the robustness of the perception manipulation funnel for uncertainty elimination in the perception state space (as illustrated in Fig.~\ref{fig:system}-(b)). Specifically, we aim to demonstrate that 1) the perception manipulation funnel is robust against a large range of action uncertainties and 2) the entropy-based exploration significantly accelerates the uncertainty elimination process. 
\begin{figure*}
    \centering
    \includegraphics[width=\textwidth]{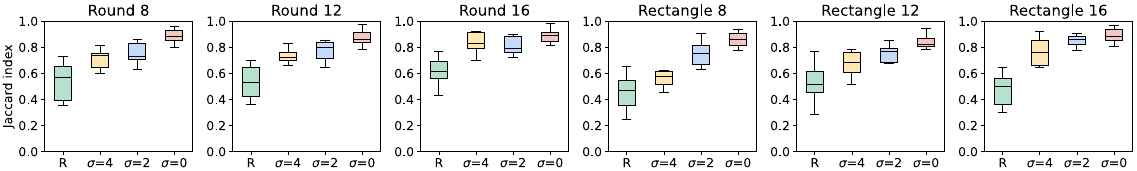}
    \caption{Experimental result of the uncertainty elimination process in simulation with different levels of action uncertainty and object scales. ``R" denotes the random policy; $\sigma$ denotes our entropy-based policy with different uncertainty levels in millimeters.}
    \label{fig:sim_pf_result}
\end{figure*}

To quantify the uncertainty level of actions, given a desired interaction point $ \mathbf{g}_d = (x_d, y_d, 0)$ on the board surface plane ($XY$-plane of $\{O\}$), the actual interacted point $\mathbf{g}_a = (x_a, x_a, 0)$ incorporates additional positional noise as action uncertainty from a 2-D Gaussian distribution as follows:
\begin{gather}    
f(x_a, y_a) = \frac{1}{2 \pi \sqrt{|\Sigma|}} \exp\left(-\frac{1}{2} \mathbf{z}^\top \Sigma^{-1} \mathbf{z}\right)  \\
\mathbf{z} = \begin{bmatrix} x_a - x_d \\ y_a - y_d \end{bmatrix} \quad \Sigma = \begin{bmatrix} 1  \quad 0 \\ 0 \quad 1 \end{bmatrix}\sigma
\end{gather}
where $\sigma$ represents the noise level of distance. We use the Jaccard index to calculate the distance between the ground truth hole area and the union of sampled areas $\{\widetilde{\mathcal{A}}\}=\{\widetilde{T}_{\{OH\},i}\mathcal{A}_h\}^{K=200}_{i=1} \sim P(\widetilde{T}_{\{OH\}})$ based on the refined proposal distribution as the convergence level:
\begin{equation}
    J(\{\widetilde{\mathcal{A}}\}) = \frac{\mid T_{\{OH\}}\mathcal{A}_h \cap \cup \{\widetilde{\mathcal{A}}\} \mid}{\mid T_{\{OH\}}\mathcal{A}_h \cup  \cup \{\widetilde{\mathcal{A}}\} \mid} \approx	 \frac{|\mathcal{A}_h|}{|\cup \{\widetilde{\mathcal{A}}\}|}
\end{equation} 
Random exploration serves as a baseline policy to demonstrate that 1) the intrinsic convergence mechanics of the perception manipulation does not rely on any specific action strategy and 2) the proposed entropy-based exploration benefits the convergence efficiency.

The results, presented in Fig.~\ref{fig:sim_pf_result}, show the system's performance after eight steps of interaction under different levels of action noises and a random policy. Pegs with smaller scales are more sensitive to the system's uncertainty, as the relative error is significantly larger than the pegs with larger scales. Despite high action noise relative to the peg's scale, all the test objects successfully converge to a high convergence level. Even random actions would significantly shrink the proposal distribution of the target hole due to the mechanics of the formulated perception funnel.

\begin{table}[]
\centering
\begin{threeparttable}          

\begin{tabular}{ |p{1.6cm}||p{1cm}<{\centering}|p{2.2cm}<{\centering}|p{1.8cm}<{\centering}|  }
 \hline
 Name & Clearance & Position-based$^1$ & Funnel-based$^2$ \\
 \hline
 Round 8  &  $\sim 0.8$   & 4/10 & \textbf{10/10}\\
 Round 12 &  $\sim 0.8$  &  5/10  & \textbf{10/10}\\
Round 16  & $\sim 0.8$ &  5/10 & \textbf{10/10} \\
Rectangle  8 & $\sim 0.6$ & 3/10 & \textbf{10/10}\\
 Rectangle 12&  $\sim 0.7$ & 5/10 & \textbf{10/10}\\
 Rectangle 16 & $\sim 0.8$  &  4/10 & \textbf{10/10}\\
 Random  \#1 & $\sim 0.4$ & \textbf{1/10} & 9/10\\
 Random  \#2&  $\sim 0.4$ & \textbf{0/10} & 8/10\\
 Random  \#3 & $\sim 0.4$  &  \textbf{0/10} & \textbf{10/10}\\
 \hline
  \hline
 Average & $\backslash$  &  \textbf{3.0/10} & \textbf{9.7/10}\\
  \hline

\end{tabular}
 \begin{tablenotes}    
        \footnotesize               
        \item[1] Baseline: Actuated with Cartesian position controller.
        \item[2] Ours: Actuated with Cartesian impedance controller.
\end{tablenotes}           
\caption{Experimental Result of Real-world Insertion Tasks}
\label{tab:exp_pf}
\end{threeparttable}    
\vspace{-0.6cm}
\end{table}

\begin{table*}[h]
    \centering
    \caption{System Performance under Different Prior Knowledge on the Target Hole}
    \label{tab:results}
    \begin{tabular}{l||ccc|ccc||c}
        \hline
        & \multicolumn{3}{c|}{\textbf{Partially Inside the Hole}} & \multicolumn{3}{c||}{\textbf{Bounded Area}} & \multicolumn{1}{c}{\textbf{Overall}} \\
        \hline
        {\textbf{Object}} & \textbf{Success} & \textbf{Interactions ($\pm$ std)} & \textbf{Uncertainty($\pm$ std)}  & \textbf{Success} & \textbf{Interactions ($\pm$ std)} & \textbf{Uncertainty($\pm$ std)} & Success\\
        \hline
         Round 8     & 5 / 5 & 4.0 $\pm$ 0.7 & 0.152 $\pm$ 0.029 & 5/5 & 4.6 $\pm$ 1.5 & 0.148 $\pm$ 0.040 & 10/10 \\
         Round 12  & 5 / 5 & 4.4 $\pm$ 1.1 & 0.155 $\pm$ 0.048 & 5/5 & 5.0 $\pm$ 1.7 & 0.150 $\pm$ 0.036 & 10/10\\
         Round 16   & 5 / 5 & 4.6 $\pm$ 0.9 & 0.163 $\pm$ 0.060 & 5/5 & 5.6 $\pm$ 1.5 & 0.162 $\pm$ 0.063 & 10/10\\
         Rectangle  8   & 5 / 5 & 6.8 $\pm$ 1.6 & 0.178 $\pm$ 0.031 & 4/5 & 7.8 $\pm$ 1.9 & 0.174 $\pm$ 0.043 & 9/10\\
         Rectangle 12    & 5 / 5 & 6.8 $\pm$ 1.3 & 0.174 $\pm$ 0.046 & 5/5 & 7.6 $\pm$ 1.5 & 0.170 $\pm$ 0.033 & 10/10\\
         Rectangle 16    & 5 / 5 & 7.2 $\pm$ 0.8 & 0.178 $\pm$ 0.031 & 5/5 & 7.8 $\pm$ 2.1 & 0.193 $\pm$ 0.055 & 10/10\\
         Random  \#1   & 5 / 5 & 9.2 $\pm$ 2.2 & 0.172 $\pm$ 0.029 & 4/5 & 9.0 $\pm$ 2.9 & 0.158 $\pm$ 0.036 & 9/10\\
         Random \#2    & 4 / 5 & 9.4 $\pm$ 2.7 & 0.164 $\pm$ 0.062 & 4/5 & 10.2 $\pm$ 1.9 & 0.181 $\pm$ 0.065 & 8/10\\
         Random \#3    & 4 / 5 & 9.0 $\pm$ 2.9 & 0.170 $\pm$ 0.045 & 5/5 & 9.6 $\pm$ 1.8 & 0.178 $\pm$ 0.061 & 9/10\\
        \hline
        \hline
        \textbf{Average}  & \textbf{4.78 / 5} & 6.8 $\pm$ 2.6 & 0.167 $\pm$ 0.041 & \textbf{4.67/5} & 7.5 $\pm$ 2.6 & 0.168 $\pm$ 0.046 & \textbf{9.44/10}\\
        \bottomrule
    \end{tabular}
\label{tab:sys_exp}
\end{table*}

\subsection{Physical Manipulation Funnel}
We evaluate the robustness of the physical manipulation funnel through peg insertion into the target hole as shown in Fig.~\ref{fig:system}-(c), which is considered a known position. We compare the robustness of funnel-based manipulation with position-based manipulation as a baseline, which formulates the insertion process by positioning the peg into the hole without collisions. The collision-free trajectory is defined as a top-down insertion with a vertical peg state after a stable grasp. Since execution uncertainties are inevitable in real-world scenarios, we aim to demonstrate that 1) when system uncertainty exceeds the allowed clearance, funnel-based manipulation for peg insertion is more robust against position-based manipulation in real-world settings, and 2) precise manipulation tasks can be achieved through imprecise funnel-based manipulation.

The experimental result is shown in Table~\ref{tab:exp_pf}. Despite the tight clearance being smaller than the robot execution precision, our funnel-based insertion effectively absorbs unmodeled uncertainties and inserts the peg at a \textit{near-perfect} success rate. In contrast, position-based manipulation expects a precisely executed trajectory while lacking the ability to tolerate such uncertainty. Without the help of chamfers on the peg and the hole, it can easily fail due to the unexpected misalignments or impact caused by the wedging phenomenon~\cite{whitney1982quasi}. Especially for asymmetric pegs with a tighter tolerance in both translation and orientation, we observe that position-based control is ineffective at small scales.

\subsection{Real-world System Performance}
 Ultimately, we combined the perception and physical manipulation funnel in a unified system to perform robust insertion without visual feedback in an open loop, as illustrated in Fig.~\ref{fig:system}-(d). To comprehensively evaluate the proposed system, we consider two levels of prior on the target hole's state: 1) the initial pose of the peg is placed partially inside the target hole by kinesthetic teaching and 2) a larger search area (bounding circle with $\sim30\%$ dilation) is specified, containing the target hole on the board surface. The robot needs to identify the possible state of the target hole on the board surface to finish the insertion. The results (as presented in Table~\ref{tab:sys_exp}) depict the steps of interactions, uncertainty level ($1- J(\{\widetilde{\mathcal{A}\}})$) and the successful rate for each tested peg. Since the proposed system effectively absorbs uncertainties from an imprecise target hole state and interactions, we observe very few failure cases from the overall performance.
 
We acknowledge the gap between the established object-centric theory and real-world implementation; the failure mode can be divided into the following categories: 1) high contact force which breaks the condition of in-hand stable grasp, i.e., the gripper is opened and the peg beyond the observation range of the tactile sensor; 2) high friction force caused by interactions between the peg and the hole that causes the peg largely slipped and is out of the tactile perception range and 3) the inaccurate dynamics model of the manipulator for impedance control that leads to unsatisfied compliance performance during physical contacts. We believe that using a more forceful end-effector or a more advanced compliance mechanism would effectively resolve this issue.

\section{Limitations}
\label{sec:limits}
Our work has certain limitations, which naturally reveal the potential future research directions. First, the current system relies on a stable in-hand grasp of the peg to perform the insertion process. A re-grasping policy~\cite{calandra2018more} can be incorporated to initialize a new insertion trial for drastic slippage that exceeds the tactile perception range. Second, current exploration relies on a priori of the planar exploratory area. By combining visual modalities and learning-based estimation methods, the system can perform environmental identification and insertion without any prior. Third, as our method doesn’t model friction explicitly, a variable impedance controller is considered a requirement when encountering interaction with high resistance. 

We envision a future iteration of the proposed robotic system that leverages multi-modal perception to partially perceive the task environment and further plan funnel-based manipulation strategies to robustly shape the interaction process toward the desired outcome beyond the peg-in-hole insertion task.

\section{Conclusion and Discussion}
In this paper, we have presented a funnel-based manipulation paradigm for robust peg-in-hole insertion, along with a real-world robotic implementation. The advantages of the proposed system have been demonstrated through careful experimentation in simulation and in the real world.  A comprehensive evaluation of the proposed system is provided on a standard NIST ATB benchmark and additional challenging tasks with a satisfactory overall success rate. Most critically, we highlighted the power of compliance and demonstrated how a robot can leverage imprecise compliant interactions to perform precise tasks, similar to how humans achieve dexterous manipulation. For future work, we aim to expand the current theory to non-planar scenarios and develop automated learning methods capable of abstracting funnel-based manipulation mechanics from natural language descriptions or video demonstrations and their interaction mechanisms.

\label{sec:conclusion}

\section*{Acknowledgments}
This work was supported by the award 60NANB24D296 from U.S. Department of Commerce, National Institute of Standards and Technology.

\bibliographystyle{plainnat}
\bibliography{references}

\begin{thebibliography}{57}
\providecommand{\natexlab}[1]{#1}
\providecommand{\url}[1]{\texttt{#1}}
\expandafter\ifx\csname urlstyle\endcsname\relax
  \providecommand{\doi}[1]{doi: #1}\else
  \providecommand{\doi}{doi: \begingroup \urlstyle{rm}\Url}\fi

\bibitem[Akella et~al.(2000)Akella, Huang, Lynch, and Mason]{akella2000parts}
Srinivas Akella, Wesley~H Huang, Kevin~M Lynch, and Matthew~T Mason.
\newblock Parts feeding on a conveyor with a one joint robot.
\newblock \emph{Algorithmica}, 26:\penalty0 313--344, 2000.

\bibitem[Bhatt et~al.(2021)Bhatt, Sieler, Puhlmann, and Brock]{Bhatt-RSS-21}
Aditya Bhatt, Adrian Sieler, Steffen Puhlmann, and Oliver Brock.
\newblock {Surprisingly Robust In-Hand Manipulation: An Empirical Study}.
\newblock In \emph{Proceedings of Robotics: Science and Systems}, Virtual, July 2021.
\newblock \doi{10.15607/RSS.2021.XVII.089}.

\bibitem[Burridge et~al.(1999)Burridge, Rizzi, and Koditschek]{burridge1999sequential}
Robert~R Burridge, Alfred~A Rizzi, and Daniel~E Koditschek.
\newblock Sequential composition of dynamically dexterous robot behaviors.
\newblock \emph{The International Journal of Robotics Research}, 18\penalty0 (6):\penalty0 534--555, 1999.

\bibitem[Calandra et~al.(2018)Calandra, Owens, Jayaraman, Lin, Yuan, Malik, Adelson, and Levine]{calandra2018more}
Roberto Calandra, Andrew Owens, Dinesh Jayaraman, Justin Lin, Wenzhen Yuan, Jitendra Malik, Edward~H Adelson, and Sergey Levine.
\newblock More than a feeling: Learning to grasp and regrasp using vision and touch.
\newblock \emph{IEEE Robotics and Automation Letters}, 3\penalty0 (4):\penalty0 3300--3307, 2018.

\bibitem[Canberk et~al.(2023)Canberk, Chi, Ha, Burchfiel, Cousineau, Feng, and Song]{canberk2023cloth}
Alper Canberk, Cheng Chi, Huy Ha, Benjamin Burchfiel, Eric Cousineau, Siyuan Feng, and Shuran Song.
\newblock Cloth funnels: Canonicalized-alignment for multi-purpose garment manipulation.
\newblock In \emph{2023 IEEE International Conference on Robotics and Automation (ICRA)}, pages 5872--5879. IEEE, 2023.

\bibitem[Chen et~al.(2024)Chen, Bohg, and Liu]{Chen-RSS-24}
Sirui Chen, Jeannette Bohg, and Karen Liu.
\newblock {SpringGrasp: Synthesizing Compliant, Dexterous Grasps under Shape Uncertainty}.
\newblock In \emph{Proceedings of Robotics: Science and Systems}, Delft, Netherlands, July 2024.
\newblock \doi{10.15607/RSS.2024.XX.042}.

\bibitem[Chen et~al.(2023)Chen, Tekden, Deisenroth, and Bekiroglu]{chen2023sliding}
Yiting Chen, Ahmet~Ercan Tekden, Marc~Peter Deisenroth, and Yasemin Bekiroglu.
\newblock Sliding touch-based exploration for modeling unknown object shape with multi-fingered hands.
\newblock In \emph{2023 IEEE/RSJ International Conference on Intelligent Robots and Systems (IROS)}, pages 8943--8950. IEEE, 2023.

\bibitem[Cheng et~al.(2022)Cheng, Huang, Hou, and Mason]{cheng2022contact}
Xianyi Cheng, Eric Huang, Yifan Hou, and Matthew~T Mason.
\newblock Contact mode guided motion planning for quasidynamic dexterous manipulation in 3d.
\newblock In \emph{2022 International Conference on Robotics and Automation (ICRA)}, pages 2730--2736. IEEE, 2022.

\bibitem[Christiansen(1991)]{christiansen1991manipulation}
Alan~D Christiansen.
\newblock Manipulation planning for empirical backprojections.
\newblock In \emph{Proceedings. 1991 IEEE International Conference on Robotics and Automation}, pages 762--763. IEEE Computer Society, 1991.

\bibitem[Coumans and Bai(2016)]{coumans2016pybullet}
Erwin Coumans and Yunfei Bai.
\newblock Pybullet, a python module for physics simulation for games, robotics and machine learning, 2016.

\bibitem[Dafle et~al.(2014)Dafle, Rodriguez, Paolini, Tang, Srinivasa, Erdmann, Mason, Lundberg, Staab, and Fuhlbrigge]{dafle2014extrinsic}
Nikhil~Chavan Dafle, Alberto Rodriguez, Robert Paolini, Bowei Tang, Siddhartha~S Srinivasa, Michael Erdmann, Matthew~T Mason, Ivan Lundberg, Harald Staab, and Thomas Fuhlbrigge.
\newblock Extrinsic dexterity: In-hand manipulation with external forces.
\newblock In \emph{2014 IEEE International Conference on Robotics and Automation (ICRA)}, pages 1578--1585. IEEE, 2014.

\bibitem[Deimel and Brock(2016)]{deimel2016novel}
Raphael Deimel and Oliver Brock.
\newblock A novel type of compliant and underactuated robotic hand for dexterous grasping.
\newblock \emph{The International Journal of Robotics Research}, 35\penalty0 (1-3):\penalty0 161--185, 2016.

\bibitem[Dong et~al.(2021)Dong, Jha, Romeres, Kim, Nikovski, and Rodriguez]{dong2021tactile}
Siyuan Dong, Devesh~K Jha, Diego Romeres, Sangwoon Kim, Daniel Nikovski, and Alberto Rodriguez.
\newblock Tactile-rl for insertion: Generalization to objects of unknown geometry.
\newblock In \emph{2021 IEEE International Conference on Robotics and Automation (ICRA)}, pages 6437--6443. IEEE, 2021.

\bibitem[Eppner and Brock(2015)]{eppner2015planning}
Clemens Eppner and Oliver Brock.
\newblock Planning grasp strategies that exploit environmental constraints.
\newblock In \emph{2015 IEEE international conference on robotics and automation (ICRA)}, pages 4947--4952. IEEE, 2015.

\bibitem[Eppner et~al.(2015)Eppner, Deimel, Alvarez-Ruiz, Maertens, and Brock]{eppner2015exploitation}
Clemens Eppner, Raphael Deimel, Jos{\'e} Alvarez-Ruiz, Marianne Maertens, and Oliver Brock.
\newblock Exploitation of environmental constraints in human and robotic grasping.
\newblock \emph{The International Journal of Robotics Research}, 34\penalty0 (7):\penalty0 1021--1038, 2015.

\bibitem[Erdmann(1985)]{erdmann1985using}
Michael Erdmann.
\newblock Using backprojections for fine motion planning with uncertainty.
\newblock In \emph{Proceedings. 1985 IEEE International Conference on Robotics and Automation}, volume~2, pages 549--554. IEEE, 1985.

\bibitem[Erdmann and Mason(1988)]{erdmann1988exploration}
Michael~A Erdmann and Matthew~T Mason.
\newblock An exploration of sensorless manipulation.
\newblock \emph{IEEE Journal on Robotics and Automation}, 4\penalty0 (4):\penalty0 369--379, 1988.

\bibitem[Goldberg(1993)]{goldberg1993orienting}
Kenneth~Y Goldberg.
\newblock Orienting polygonal parts without sensors.
\newblock \emph{Algorithmica}, 10\penalty0 (2):\penalty0 201--225, 1993.

\bibitem[Hang et~al.(2016)Hang, Li, Stork, Bekiroglu, Pokorny, Billard, and Kragic]{hang2016hierarchical}
Kaiyu Hang, Miao Li, Johannes~A Stork, Yasemin Bekiroglu, Florian~T Pokorny, Aude Billard, and Danica Kragic.
\newblock Hierarchical fingertip space: A unified framework for grasp planning and in-hand grasp adaptation.
\newblock \emph{IEEE Transactions on robotics}, 32\penalty0 (4):\penalty0 960--972, 2016.

\bibitem[Hang et~al.(2019)Hang, Morgan, and Dollar]{hang2019pre}
Kaiyu Hang, Andrew~S Morgan, and Aaron~M Dollar.
\newblock Pre-grasp sliding manipulation of thin objects using soft, compliant, or underactuated hands.
\newblock \emph{IEEE Robotics and Automation Letters}, 4\penalty0 (2):\penalty0 662--669, 2019.

\bibitem[Hang et~al.(2021)Hang, Bircher, Morgan, and Dollar]{hang2021manipulation}
Kaiyu Hang, Walter~G Bircher, Andrew~S Morgan, and Aaron~M Dollar.
\newblock Manipulation for self-identification, and self-identification for better manipulation.
\newblock \emph{Science robotics}, 6\penalty0 (54):\penalty0 eabe1321, 2021.

\bibitem[Haugaard et~al.(2021)Haugaard, Langaa, Sloth, and Buch]{haugaard2021fast}
Rasmus Haugaard, Jeppe Langaa, Christoffer Sloth, and Anders Buch.
\newblock Fast robust peg-in-hole insertion with continuous visual servoing.
\newblock In \emph{Conference on Robot Learning}, pages 1696--1705. PMLR, 2021.

\bibitem[Hogan(1984)]{hogan1984impedance}
Neville Hogan.
\newblock Impedance control: An approach to manipulation.
\newblock In \emph{1984 American control conference}, pages 304--313. IEEE, 1984.

\bibitem[Hou et~al.(2020)Hou, Jia, and Mason]{hou2020manipulation}
Yifan Hou, Zhenzhong Jia, and Matthew~T Mason.
\newblock Manipulation with shared grasping.
\newblock \emph{arXiv preprint arXiv:2006.02996}, 2020.

\bibitem[Inoue et~al.(2017)Inoue, De~Magistris, Munawar, Yokoya, and Tachibana]{inoue2017deep}
Tadanobu Inoue, Giovanni De~Magistris, Asim Munawar, Tsuyoshi Yokoya, and Ryuki Tachibana.
\newblock Deep reinforcement learning for high precision assembly tasks.
\newblock In \emph{2017 IEEE/RSJ International Conference on Intelligent Robots and Systems (IROS)}, pages 819--825. IEEE, 2017.

\bibitem[Jin et~al.(2021)Jin, Zhu, Wang, and Tomizuka]{jin2021contact}
Shiyu Jin, Xinghao Zhu, Changhao Wang, and Masayoshi Tomizuka.
\newblock Contact pose identification for peg-in-hole assembly under uncertainties.
\newblock In \emph{2021 American Control Conference (ACC)}, pages 48--53. IEEE, 2021.

\bibitem[Khadivar and Billard(2023)]{khadivar2023adaptive}
Farshad Khadivar and Aude Billard.
\newblock Adaptive fingers coordination for robust grasp and in-hand manipulation under disturbances and unknown dynamics.
\newblock \emph{IEEE Transactions on Robotics}, 39\penalty0 (5):\penalty0 3350--3367, 2023.

\bibitem[Kim and Rodriguez(2022)]{kim2022active}
Sangwoon Kim and Alberto Rodriguez.
\newblock Active extrinsic contact sensing: Application to general peg-in-hole insertion.
\newblock In \emph{2022 International Conference on Robotics and Automation (ICRA)}, pages 10241--10247. IEEE, 2022.

\bibitem[Kimble et~al.(2020)Kimble, Van~Wyk, Falco, Messina, Sun, Shibata, Uemura, and Yokokohji]{kimble2020benchmarking}
Kenneth Kimble, Karl Van~Wyk, Joe Falco, Elena Messina, Yu~Sun, Mizuho Shibata, Wataru Uemura, and Yasuyoshi Yokokohji.
\newblock Benchmarking protocols for evaluating small parts robotic assembly systems.
\newblock \emph{IEEE robotics and automation letters}, 5\penalty0 (2):\penalty0 883--889, 2020.

\bibitem[Lauri et~al.(2022)Lauri, Hsu, and Pajarinen]{lauri2022partially}
Mikko Lauri, David Hsu, and Joni Pajarinen.
\newblock Partially observable markov decision processes in robotics: A survey.
\newblock \emph{IEEE Transactions on Robotics}, 39\penalty0 (1):\penalty0 21--40, 2022.

\bibitem[Levine et~al.(2016)Levine, Finn, Darrell, and Abbeel]{levine2016end}
Sergey Levine, Chelsea Finn, Trevor Darrell, and Pieter Abbeel.
\newblock End-to-end training of deep visuomotor policies.
\newblock \emph{Journal of Machine Learning Research}, 17\penalty0 (39):\penalty0 1--40, 2016.

\bibitem[Li et~al.(2014{\natexlab{a}})Li, Bekiroglu, Kragic, and Billard]{li2014learning_adap}
Miao Li, Yasemin Bekiroglu, Danica Kragic, and Aude Billard.
\newblock Learning of grasp adaptation through experience and tactile sensing.
\newblock In \emph{2014 IEEE/RSJ International Conference on Intelligent Robots and Systems}, pages 3339--3346. Ieee, 2014{\natexlab{a}}.

\bibitem[Li et~al.(2014{\natexlab{b}})Li, Yin, Tahara, and Billard]{li2014learning}
Miao Li, Hang Yin, Kenji Tahara, and Aude Billard.
\newblock Learning object-level impedance control for robust grasping and dexterous manipulation.
\newblock In \emph{2014 IEEE International Conference on Robotics and Automation (ICRA)}, pages 6784--6791. IEEE, 2014{\natexlab{b}}.

\bibitem[Li et~al.(2016)Li, Hang, Kragic, and Billard]{li2016dexterous}
Miao Li, Kaiyu Hang, Danica Kragic, and Aude Billard.
\newblock Dexterous grasping under shape uncertainty.
\newblock \emph{Robotics and Autonomous Systems}, 75:\penalty0 352--364, 2016.

\bibitem[Lozano-Perez et~al.(1984)Lozano-Perez, Mason, and Taylor]{lozano1984automatic}
Tomas Lozano-Perez, Matthew~T Mason, and Russell~H Taylor.
\newblock Automatic synthesis of fine-motion strategies for robots.
\newblock \emph{The International Journal of Robotics Research}, 3\penalty0 (1):\penalty0 3--24, 1984.

\bibitem[Luo et~al.(2018)Luo, Solowjow, Wen, Ojea, and Agogino]{luo2018deep}
Jianlan Luo, Eugen Solowjow, Chengtao Wen, Juan~Aparicio Ojea, and Alice~M Agogino.
\newblock Deep reinforcement learning for robotic assembly of mixed deformable and rigid objects.
\newblock In \emph{2018 IEEE/RSJ International Conference on Intelligent Robots and Systems (IROS)}, pages 2062--2069. IEEE, 2018.

\bibitem[Luo et~al.(2019)Luo, Solowjow, Wen, Ojea, Agogino, Tamar, and Abbeel]{luo2019reinforcement}
Jianlan Luo, Eugen Solowjow, Chengtao Wen, Juan~Aparicio Ojea, Alice~M Agogino, Aviv Tamar, and Pieter Abbeel.
\newblock Reinforcement learning on variable impedance controller for high-precision robotic assembly.
\newblock In \emph{2019 International Conference on Robotics and Automation (ICRA)}, pages 3080--3087. IEEE, 2019.

\bibitem[Luo et~al.(2021)Luo, Sushkov, Pevceviciute, Lian, Su, Vecerik, Ye, Schaal, and Scholz]{luo2021robust}
Jianlan Luo, Oleg Sushkov, Rugile Pevceviciute, Wenzhao Lian, Chang Su, Mel Vecerik, Ning Ye, Stefan Schaal, and Jon Scholz.
\newblock Robust multi-modal policies for industrial assembly via reinforcement learning and demonstrations: A large-scale study.
\newblock \emph{arXiv preprint arXiv:2103.11512}, 2021.

\bibitem[Luo et~al.(2024)Luo, Hu, Xu, Tan, Berg, Sharma, Schaal, Finn, Gupta, and Levine]{luo2024serl}
Jianlan Luo, Zheyuan Hu, Charles Xu, You~Liang Tan, Jacob Berg, Archit Sharma, Stefan Schaal, Chelsea Finn, Abhishek Gupta, and Sergey Levine.
\newblock Serl: A software suite for sample-efficient robotic reinforcement learning.
\newblock \emph{arXiv preprint arXiv:2401.16013}, 2024.

\bibitem[Mason(1985)]{mason1985mechanics}
Matthew Mason.
\newblock The mechanics of manipulation.
\newblock In \emph{Proceedings. 1985 IEEE International Conference on Robotics and Automation}, volume~2, pages 544--548. IEEE, 1985.

\bibitem[Morgan et~al.(2021)Morgan, Wen, Liang, Boularias, Dollar, and Bekris]{morgan2021vision}
Andrew~S Morgan, Bowen Wen, Junchi Liang, Abdeslam Boularias, Aaron~M Dollar, and Kostas Bekris.
\newblock Vision-driven compliant manipulation for reliable, high-precision assembly tasks.
\newblock \emph{arXiv preprint arXiv:2106.14070}, 2021.

\bibitem[Morgan et~al.(2022)Morgan, Hang, Wen, Bekris, and Dollar]{morgan2022complex}
Andrew~S Morgan, Kaiyu Hang, Bowen Wen, Kostas Bekris, and Aaron~M Dollar.
\newblock Complex in-hand manipulation via compliance-enabled finger gaiting and multi-modal planning.
\newblock \emph{IEEE Robotics and Automation Letters}, 7\penalty0 (2):\penalty0 4821--4828, 2022.

\bibitem[Odhner et~al.(2014)Odhner, Jentoft, Claffee, Corson, Tenzer, Ma, Buehler, Kohout, Howe, and Dollar]{odhner2014compliant}
Lael~U Odhner, Leif~P Jentoft, Mark~R Claffee, Nicholas Corson, Yaroslav Tenzer, Raymond~R Ma, Martin Buehler, Robert Kohout, Robert~D Howe, and Aaron~M Dollar.
\newblock A compliant, underactuated hand for robust manipulation.
\newblock \emph{The International Journal of Robotics Research}, 33\penalty0 (5):\penalty0 736--752, 2014.

\bibitem[Peternel et~al.(2018)Peternel, Tsagarakis, Caldwell, and Ajoudani]{peternel2018robot}
Luka Peternel, Nikos Tsagarakis, Darwin Caldwell, and Arash Ajoudani.
\newblock Robot adaptation to human physical fatigue in human--robot co-manipulation.
\newblock \emph{Autonomous Robots}, 42:\penalty0 1011--1021, 2018.

\bibitem[Posa et~al.(2014)Posa, Cantu, and Tedrake]{posa2014direct}
Michael Posa, Cecilia Cantu, and Russ Tedrake.
\newblock A direct method for trajectory optimization of rigid bodies through contact.
\newblock \emph{The International Journal of Robotics Research}, 33\penalty0 (1):\penalty0 69--81, 2014.

\bibitem[Rodriguez(2021)]{rodriguez2021unstable}
Alberto Rodriguez.
\newblock The unstable queen: Uncertainty, mechanics, and tactile feedback.
\newblock \emph{Science Robotics}, 6\penalty0 (54):\penalty0 eabi4667, 2021.

\bibitem[Shao et~al.(2020)Shao, Migimatsu, and Bohg]{shao2020learning}
Lin Shao, Toki Migimatsu, and Jeannette Bohg.
\newblock Learning to scaffold the development of robotic manipulation skills.
\newblock In \emph{2020 IEEE International Conference on Robotics and Automation (ICRA)}, pages 5671--5677. IEEE, 2020.

\bibitem[Shao et~al.(2024)Shao, Li, Keyvanian, Chaudhari, Kumar, and Figueroa]{Shao-RSS-24}
Yifei~Simon Shao, Tianyu Li, Shafagh Keyvanian, Pratik Chaudhari, Vijay Kumar, and Nadia Figueroa.
\newblock {Constraint-Aware Intent Estimation for Dynamic Human-Robot Object Co-Manipulation}.
\newblock In \emph{Proceedings of Robotics: Science and Systems}, Delft, Netherlands, July 2024.
\newblock \doi{10.15607/RSS.2024.XX.028}.

\bibitem[Shi et~al.(1994)]{shi1994good}
Jianbo Shi et~al.
\newblock Good features to track.
\newblock In \emph{1994 Proceedings of IEEE conference on computer vision and pattern recognition}, pages 593--600. IEEE, 1994.

\bibitem[Simunovi{\"A}~Simunovi{\"A}(1979)]{simunovia1979information}
Sergio~Natalio Simunovi{\"A}~Simunovi{\"A}.
\newblock \emph{An information approach to parts mating}.
\newblock PhD thesis, Massachusetts Institute of Technology, 1979.

\bibitem[Tang et~al.(2024)Tang, Akinola, Xu, Wen, Handa, Wyk, Fox, Sukhatme, Ramos, and Narang]{Tang-RSS-24}
Bingjie Tang, Iretiayo Akinola, Jie Xu, Bowen Wen, Ankur Handa, Karl~Van Wyk, Dieter Fox, Gaurav~S. Sukhatme, Fabio Ramos, and Yashraj Narang.
\newblock {AutoMate: Specialist and Generalist Assembly Policies over Diverse Geometries}.
\newblock In \emph{Proceedings of Robotics: Science and Systems}, Delft, Netherlands, July 2024.
\newblock \doi{10.15607/RSS.2024.XX.064}.

\bibitem[Tang et~al.(2016)Tang, Lin, Zhao, Chen, and Tomizuka]{tang2016autonomous}
Te~Tang, Hsien-Chung Lin, Yu~Zhao, Wenjie Chen, and Masayoshi Tomizuka.
\newblock Autonomous alignment of peg and hole by force/torque measurement for robotic assembly.
\newblock In \emph{2016 IEEE international conference on automation science and engineering (CASE)}, pages 162--167. IEEE, 2016.

\bibitem[Toussaint et~al.(2014)Toussaint, Ratliff, Bohg, Righetti, Englert, and Schaal]{toussaint2014dual}
Marc Toussaint, Nathan Ratliff, Jeannette Bohg, Ludovic Righetti, Peter Englert, and Stefan Schaal.
\newblock Dual execution of optimized contact interaction trajectories.
\newblock In \emph{2014 IEEE/RSJ International Conference on Intelligent Robots and Systems}, pages 47--54. IEEE, 2014.

\bibitem[Whitney et~al.(1982)]{whitney1982quasi}
Daniel~E Whitney et~al.
\newblock Quasi-static assembly of compliantly supported rigid parts.
\newblock \emph{Journal of Dynamic Systems, Measurement, and Control}, 104\penalty0 (1):\penalty0 65--77, 1982.

\bibitem[Xie et~al.(2022)Xie, Yu, Zhao, Zhang, Zhou, Wang, Wang, and Xiong]{xie2022learning}
Liang Xie, Hongxiang Yu, Yinghao Zhao, Haodong Zhang, Zhongxiang Zhou, Minhang Wang, Yue Wang, and Rong Xiong.
\newblock Learning to fill the seam by vision: Sub-millimeter peg-in-hole on unseen shapes in real world.
\newblock In \emph{2022 International conference on robotics and automation (ICRA)}, pages 2982--2988. IEEE, 2022.

\bibitem[Zhang et~al.(2021)Zhang, Sun, Kuang, and Tomizuka]{zhang2021learning}
Xiang Zhang, Liting Sun, Zhian Kuang, and Masayoshi Tomizuka.
\newblock Learning variable impedance control via inverse reinforcement learning for force-related tasks.
\newblock \emph{IEEE Robotics and Automation Letters}, 6\penalty0 (2):\penalty0 2225--2232, 2021.

\bibitem[Zhou and Held(2023)]{zhou2023learning}
Wenxuan Zhou and David Held.
\newblock Learning to grasp the ungraspable with emergent extrinsic dexterity.
\newblock In \emph{Conference on Robot Learning}, pages 150--160. PMLR, 2023.

\end{thebibliography}

\newpage

\section*{Appendix}
\label{appendix}
\subsection{Cartesian Impedance Controller}
\label{app:imp}
The state of Franka Emika Panda in the configuration space at time $t$ is denoted as $\mathbf{q}_t\in \mathbb{R}^7$, with its velocity as $\dot{\mathbf{q}}_t \in \mathbb{R}^7$ and acceleration as $\Ddot{\mathbf{q}}_t \in \mathbb{R}^7$. Given the desired state $\mathbf{x}^{\text{d}}_t$ and current state $\mathbf{x}_t$ of the manipulated frame, the robot motion is regulated by the Cartesian impedance controller based on the simulated force $\mathbf{F}$ as:
\begin{equation}
    \mathbf{M}(\mathbf{q}_t)\Ddot{\mathbf{q}_t} + \mathbf{C}(\mathbf{q}_t, \dot{\mathbf{q}}_t)\dot{\mathbf{q}}_t + \mathbf{g}(\mathbf{q}_t) = \mathbf{J}(\mathbf{q}_t)^\top \mathbf{F}
\end{equation}
where $\mathbf{M}(\mathbf{q}_t)\in \mathbb{R}^{7\times7}$ is the inertial matrix, $\mathbf{C}(\mathbf{q}_t, \dot{\mathbf{q}}_t)$ is the Coriolis and centrifugal matrix, $\mathbf{g}(\mathbf{q}_t)$ is the gravity vector and $\mathbf{J}(\mathbf{q}_t)^\top \in \mathbb{R}^{7\times6}$ is the transpose of the Jacobian matrix.

\subsection{Tactile Pose Estimation}
\label{app:inhand}
\begin{figure}[H]
    \centering
    \includegraphics[width=\linewidth]{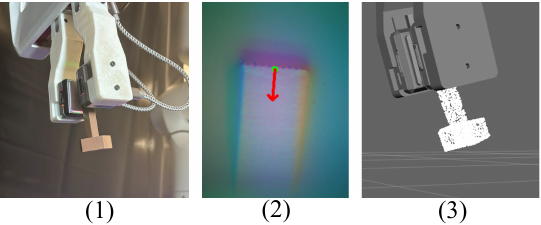}
    \caption{(1) A stable grasp of the peg for insertion; (2) The corresponding tactile image with the object's imprint; (3) The estimated in-hand pose of the peg.}
    \label{fig:inhand}
\vspace{-0.2cm}
\end{figure}

We estimate the peg's in-hand pose based on its tail end's tactile imprint (as shown in Fig.~\ref{fig:inhand}). After a stable grasping, the motion of the object is limited in the SE(2) space as  $[x,y,\omega]$ on the planar gel surface. We identify the keypoints from the corner of the peg in the tactile image with a gradient-based feature extraction method~\cite{shi1994good}. As the connecting segment between the extracted corners represents the base of the peg, the peg's position is defined by its center, and the peg's orientation is perpendicular to the formed segment.
To alleviate the perception noise for a stable estimation, we apply an Extended Kalman Filter (EKF) to smooth the estimated result and only take it into consideration when it reaches stability.
As long as the peg after slippage does not exceed the observation range of the fingertip tactile sensors, the estimated in-hand pose is further leveraged for object-centric manipulation planning. For our experiments, the perception range of the Gelsight Mini tactile sensor is around $25mm\times20mm$, and the end of the peg is required to show up completely.

\subsection{System Transition Function for MPC}
\label{app:trans}
We model the system transition function with a linear prior, as if there is no external contact, the steady state will reach the desired state during each interaction. The discrete linear state-space model of the system is formulated as:
\begin{equation}
\label{eq:linear}
    \mathbf{x}_{t+1} = \mathbf{A}_t\mathbf{x}_{t} + \mathbf{B}_t\mathbf{u}_t \text{ where } \mathbf{u}_t = \mathbf{x}_t^{\text{d}} - \mathbf{x}_t
\end{equation}
where $\mathbf{A}\in\mathbb{R}^{6\times6}$ is the state transition matrix and $\mathbf{B} \in \mathbb{R}^{6\times6}$ is the input matrix. Both $\mathbf{A}$ and $\mathbf{B}$ are initialized as an Identity matrix.

However, under the un-modeled contact after corner alignment, the steady state of the peg is the result of the interaction between the potential from the impedance control and its environmental constraints. We aim to approximate the intricate transition function online, as $\mathbf{A}$ and $\mathbf{B}$ from Eq.~\eqref{eq:linear} are updated by an RLS adaptive filter. The system dynamics is rewritten as:
\begin{equation}
\begin{aligned}
    \mathbf{x}_{t+1} = &\Phi_t \gamma_t +\epsilon_t, \text{ where } \epsilon_t\in \mathbb{R}^6\\
    \Phi_t = \begin{bmatrix}
        \mathbf{x}_t \\ \mathbf{u}_t 
    \end{bmatrix}^\top \in \mathbb{R}^{6\times12}, & \quad \gamma_t = \begin{bmatrix}
        \mathbf{A}_t \quad \mathbf{B}_t
    \end{bmatrix} \in \mathbb{R}^{12\times6}
\end{aligned}
\end{equation}
$\Phi_t$ is the regression vector, $\gamma_t$ is the system to update and $\epsilon_t$ is the prediction  error.
Given a covariance matrix $\mathbf{P}_t$ and a forgetting factor $\lambda$, the gain $\mathbf{K}_t$ is calculated based on the prediction error after each interaction as:
\begin{gather}
    \epsilon_t = \mathbf{x}_{t+1} - \Phi_t \gamma_t\\
    \mathbf{K}_t = \mathbf{P}_t  \Phi_t^\top(\lambda+\Phi_t\mathbf{P}_t\Phi_t^\top)^{-1}
\end{gather}
and the system parameters are updated as:
\begin{gather}
    \gamma_{t+1} = \gamma_{t}+\mathbf{K}_t\epsilon_t\\
    \gamma_{t+1} = [\mathbf{A}_{t+1} \quad \mathbf{B}_{t+1}]
\end{gather}
The covariance matrix is then updated as:
\begin{equation}
    \mathbf{P}_{t+1} = \frac{1}{\lambda}(\mathbf{P}_t - \mathbf{K}_{t}\Phi_t\mathbf{P}_t)
\end{equation}

Thus, the approximated system transition function $\widetilde{\Pi}_t(\mathbf{x}_t, \mathbf{x}_d^{\text{d}})$ is defined as:
\begin{equation}
\label{sys_tran}
    \widetilde{\Pi}_t(\mathbf{x}_t, \mathbf{x}_d^{\text{d}}) = \mathbf{A}_t\mathbf{x}_t + \mathbf{B}_t(\mathbf{x}_t^{\text{d}} - \mathbf{x}_t)
\end{equation}
and updated online to identify the underlying dynamics.

\subsection{Computational Complexity Analysis}
We perform an additional analysis of the computational efficiency of the proposed system. For Alg.~\ref{interactive_p}, we adopted a rejection-based sampling algorithm with the theoretical time complexity per accepted sample as $\mathcal{O}(M)$ ($M$ is the reciprocal of the acceptance rate conditioned on the gathered inequality constraints defined in \textit{Definition}~\ref{ineq}).
Alg.~\ref{alg:alignment} calculates the intersection of the union over a constant number of samples and thus possesses a time complexity of $\mathcal{O}(1)$. For Alg.~\ref{alg:insertion}, since we are using a linear state-space model as defined in Eq.~\eqref{sys_tran}, which leads to a convex QP problem that can be efficiently solved, the MPC planner runs at $\sim 30Hz$.

\end{document}